\documentclass[11pt]{article}
\usepackage[letterpaper,margin=1in]{geometry}

\usepackage[utf8]{inputenc}
\usepackage[T1]{fontenc}
\usepackage{lmodern}
\usepackage[
  backend=biber,
  style=alphabetic,
  citestyle=alphabetic,
  maxbibnames=99,
  sorting=nyt,
  natbib=true,
  backref=true
]{biblatex}
\DefineBibliographyStrings{english}{
  backrefpage = {Cited on page},
  backrefpages = {Cited on pages},
}

\usepackage{xcolor}
\usepackage[
  colorlinks=true,
  citecolor=green!50!black,
  linkcolor=blue!50!black,
  urlcolor=blue!50!black
]{hyperref}
\usepackage{url}
\usepackage{xurl}
\usepackage{booktabs}
\usepackage{amsfonts}
\usepackage{nicefrac}
\usepackage{microtype}
\usepackage{amsmath,amssymb,amsthm}
\usepackage{mathtools}
\usepackage{bm}
\usepackage{bbm}
\usepackage[noabbrev]{cleveref}

\newcommand{\R}{\mathbb{R}}

\newcommand{\N}{\mathbb{N}}

\DeclareSymbolFont{rsfs}{U}{rsfs}{m}{n}
\DeclareSymbolFontAlphabet{\mathscrsfs}{rsfs}

\def\VCdim{\mathrm{VCdim}}

\def\lin{{\rm lin}}

\def\Unif{{\rm Unif}}
\def\lin{{\rm lin}}

\def\cG{{\mathcal G}}

\def\cF{{\mathcal F}}

\def\cG{{\mathcal G}}

\def\poly{\mathrm{poly}}

\def\cD{{\mathcal D}}

\def\cF{{\mathcal F}}

\def\lin{{\rm lin}}

\def\ind{\mathbbm{1}}

\def\sCoT{{\sf CoT}}
\def\sete{{\sf e2e}}

\newcommand{\ete}[2]{#1^{\sete\textnormal{-}{#2}}}

\newcommand{\CoT}[2]{#1^{\sCoT\textnormal{-}{#2}}}

\newcommand{\Cons}{\textsc{Consistent}}

\def\cFlin{\mathcal{F}_{\scriptscriptstyle d,\lin}}

\usepackage{bm}

\def\vu{{\bm{u}}}
\def\vv{{\bm{v}}}
\def\vw{{\bm{w}}}
\def\vx{{\bm{x}}}

\def\vz{{\bm{z}}}

\def\sF{{\mathbb{F}}}

\def\sN{{\mathbb{N}}}

\newcommand{\ord}[1]{#1^{\mathrm{th}}}

\def\val{{\rm val}}

\newtheorem{theorem}{Theorem}
\newtheorem{corollary}[]{Corollary}
\newtheorem{proposition}[]{Proposition}
\newtheorem{claim}[]{Claim}
\newtheorem{definition}[]{Definition}
\newtheorem{observation}[]{Observation}
\newtheorem{assumption}[]{Assumption}
\newtheorem{requirement}[]{Requirement}
\newtheorem{property}{Property}
\crefname{property}{property}{properties}
\Crefname{property}{Property}{Properties}
\theoremstyle{definition}
\newtheorem{remark}[]{Remark}
\newtheorem{question}{Question}
\crefname{question}{question}{questions}
\Crefname{question}{Question}{Questions}
\crefname{assumption}{assumption}{assumptions}
\Crefname{assumption}{Assumption}{Assumptions}
\crefname{requirement}{requirement}{requirements}
\Crefname{requirement}{Requirement}{Requirements}
\crefname{observation}{observation}{observations}
\Crefname{observation}{Observation}{Observations}
\crefname{claim}{claim}{claims}
\Crefname{claim}{Claim}{Claims}

\DeclareMathOperator{\relu}{ReLU}

\title{\vspace{-0.4in}The Surprising Expressivity of a Single Neuron
with \\
Autoregressive Chain-of-Thought}

\author{%
  Roey Magen\textsuperscript{1}\thanks{Corresponding author: \texttt{roey.magen@weizmann.ac.il}.}\quad
  Nirmit Joshi\textsuperscript{2}\quad
  Gal Vardi\textsuperscript{1}\\[0.6em]
  \small\textsuperscript{1}Weizmann Institute of Science\\[0.4em]
  \small\textsuperscript{2}Toyota Technological Institute at Chicago%
}
\date{}

\begin{document}

\maketitle
\vspace{-0.3in}
% \nirmit{On the Universal/Surprising Expressive Power of a Single Neuron with Autoregressive Chain-of-Thought}
% \nirmit{On the Universality of a Single Neuron with Autoregressive Chain-of-Thought}
% \nirmit{... of a Single Autoregressive Neuron}
\begingroup
\setlength{\leftmargini}{\dimexpr(\textwidth-5in)/2\relax}
\begin{abstract}

\noindent The success of autoregressive language models trained with Chain-of-Thought~(CoT) supervision raises a fundamental question: 
How simple can a next-token predictor be while still enabling the efficient simulation of general computation through its autoregressively generated CoT? 
Surprisingly, we show that a single neuron suffices. For a broad family of activations, including threshold and ReLU, a single autoregressive neuron can simulate any feedforward circuit with the same activation. 
Specifically, a size-$s$ circuit on $n$ inputs can be simulated by a neuron of dimension $O(ns+s^2)$ using $O(s^2)$ generation steps, with both bounds independent of circuit depth and tight under natural assumptions. 
For threshold activations, this improves upon the simulation of Joshi et al. (COLT 2025), whose bounds scale exponentially with circuit depth. 
Consequently, autoregressive threshold neurons can simulate polynomial-time Turing machines, while remaining efficiently learnable from CoT supervision. 
Our main expressivity result relies on a novel connection to Sidon sets from additive combinatorics.
As an experimental illustration, we efficiently train a single autoregressive ReLU neuron to mimic a neural network teacher with nontrivial test accuracy on CIFAR-10.
\end{abstract}
\endgroup

\section{Introduction}
Autoregressive (AR) generation has emerged as one of the most successful paradigms
in modern machine learning and underlies contemporary large language models
(LLMs). At each generation step, the same parameterized next-token predictor
is applied to the current context. This parameter sharing makes the model
\emph{time-invariant}, while repeated application allows a fixed predictor to perform computations of varying length.
On multi-step reasoning tasks, the generated sequence can be viewed as a Chain-of-Thought (CoT) that guides the LLM toward the correct final answer \citep{wei2022chain,kojima2022large}. From a theoretical perspective, this raises a basic expressivity question: 
% \begin{question}[Expressivity]\label{ques:expressivity}
% What is the simplest predictor that, when applied autoregressively, can efficiently\footnote{Here, ``efficiently'' means that the autoregressive model dimension and number of generation steps are polynomial in the input length and in the size (or runtime) of the simulated computation.} simulate general computation, such as that performed by Turing machines or deep neural networks?
% %What is the simplest predictor, when applied autoregressively, can simulate efficiently \footnote{Here, ``efficiently'' means that the autoregressive model dimension and number of generation steps are polynomial in the input length and in the size (or running time) of the simulated computation.} general computation such as turing machine or deep neural networks? 
% \end{question}

% \nirmit{Maybe you can say "Can a simple class of next token predictors''. So class rather than predictor, and rather than simplest, maybe the simplest is made a simple, which is mild.} \roey{What about the following?}
\begin{question}[Expressivity]\label{ques:expressivity}
How simple can a class of next-token predictors be while still enabling efficient autoregressive simulation of general 
computation,\footnote{Here, ``efficient'' means that both the number of parameters of the next-token predictor and the number of generation steps are polynomial in the input length and in the size (or runtime) of the simulated computation.}
such as that performed by Turing machines or deep neural networks?
\end{question}
Chains of thought are relevant not only at inference time but also during
training. Pretraining corpora often contain worked solutions, mathematical
derivations, and other step-by-step explanations. During post-training, and
particularly during supervised fine-tuning (SFT), reasoning traces 
%can be
are explicitly incorporated.
%included deliberately. 
These traces may consist of human-written step-by-step solutions or synthetic reasoning generated by a more capable teacher model. 
Thus, reasoning trajectories constitute a natural form of supervision: each trajectory provides a sequence of prefix-next-token examples that are used to learn a single next-token predictor. This raises the following learnability question:

\begin{question}[Learnability]\label{ques:leanability}
Can repeated application of a simple predictor simulate general computation, while remaining efficiently learnable from CoT supervision?
% \footnote{By ``complete $\sCoT$ supervision,'' we mean that, for each training input, the learner observes the entire computation trajectory, including every intermediate step, rather than only the final output.}
\end{question}

Expressivity limits cannot be overcome by better supervision,  
while expressivity alone says little about whether the model can be learned efficiently from data.

% Several works have studied autoregressive models with base classes, such as RNNs and transformers, that are already expressive and not efficiently learnable \citep{wen2025rnns,merrill2024expressive,brosamle2026expressive}.

% This motivates a large body of work on \emph{time-invariant} autoregressive models (i.e., shared parameter setting).  
% Such models have statistical advantages and enjoy sample-complexity bounds independent of the
% generation length \citep{joshi2025theory,hanneke2026sample,li2026optimal}.
% The most related work to ours is of \citet{joshi2025theory}, which formalizes both learning from
% complete $\sCoT$ trajectories and end-to-end learning from only the final answer.
% They establish general statistical and computational learning results and instantiate their framework with a simple concrete base class:
% linear thresholds. For this class, they establish efficient learning
% from complete $\sCoT$ supervision and give a time-invariant simulation of
% threshold circuits. However, their construction is not efficient and requires dimension and generation length that grow exponentially with the circuit depth.
% They explicitly leave open whether this simulation can be improved. 
% In particular, their results do not suggest that such a simple time-invariant linear-threshold predictor could \emph{efficiently} simulate turing machines.

\paragraph{Prior work and the gap addressed by this work.} 
In practice, the most widely used class of next-token predictor is that of transformers \citep{vaswani2017attention}.
Consequently, a growing line of work has established strong expressivity results for transformers augmented with AR chain-of-thought generation \citep{li2024chain,merrill2024expressive,brosamle2026expressive}. 
These results, however, do not identify the simplest next-token operations that suffice to simulate general computation.
Moreover, we focus on classes that admit provable polynomial-time learning guarantees for the next-token prediction problem.

\citet{malach2023auto} argued that much of the power of autoregressive models comes from CoT computation rather than the specific architecture. 
He supported this empirically on simple benchmarks and theoretically for AR linear-threshold models,\footnote{The paper formulates these models using argmax decoding over linear scores, which can equivalently be represented by a single linear-threshold gate.} 
with one significant caveat: his construction allows a \emph{different} linear-threshold predictor (neuron) for each generation step.
This departs from practical autoregressive architectures and makes the simulation straightforward by allowing parameters to be assigned separately to each step.

\citet{joshi2025theory} subsequently formalized the \emph{time-invariant} setting, in which a single predictor is reused across all generation steps, 
emphasizing its naturality and the statistical benefits of parameter sharing. 
These benefits include sample-complexity bounds independent of the generation length \citep{hanneke2026sample,li2026optimal}. 
However, parameter sharing complicates even basic expressivity questions. Indeed, the construction of \citet{joshi2025theory} for simulating threshold circuits with an AR linear-threshold predictor incurs an exponential dependence on circuit depth, 
requiring both dimension and generation length $O(\mathrm{size}^{\mathrm{depth}})$. 
% They explicitly leave open whether this simulation can be improved. 
% Thus, it remains open whether a simple class of next-token predictors can both efficiently simulate general computation and admit efficient learning guarantees in the time-invariant setting.
They explicitly leave open whether their simulation can be made more efficient. More broadly, it remains open whether a simple class of next-token predictors can efficiently simulate general computation while admitting efficient learning guarantees in the time-invariant setting.

\paragraph{Our results.}
We answer both the \nameref{ques:expressivity} and \nameref{ques:leanability} questions in discrete and real-valued settings:
repeated application of a \emph{single neuron} suffices. 
Our main result applies to a broad family of activations and is stated in terms of
$\sigma$-circuits, or equivalently, feedforward neural networks with activation $\sigma$:

\begin{theorem}[Informal version of \Cref{thm:LT-expressivity-tight,thm:boolean-LT-expressivity}]\label{thm:main-informal}

Under mild assumptions on the activation $\sigma$ (satisfied by threshold, ReLU, GELU, and sigmoid, among others), 
any size-$s$ $\sigma$-circuit on inputs of dimension $n$ can be simulated by autoregressively 
applying a single fixed neuron
$\vx\mapsto\sigma(\langle\vw,\vx\rangle)$,\footnote{For vectors of different dimensions, the inner product is defined by padding the shorter vector with zeros on the left.} with $\vw\in\R^d$.
The dimension of the AR neuron is $d=O(ns+s^2)$ and the number of generation steps is $T=O(s^2)$, both completely independent of
the circuit depth.
\end{theorem}

Our simulation leads to significant expressivity and learning guarantees.
We complement these guarantees by studying the optimality of the simulation bounds and providing an experimental illustration.
We present these contributions in the following order:

\begin{itemize}
     
\item \textbf{Expressivity (\Cref{sec:expressivity}).}
% For threshold activations, a single time-invariant
% AR neuron can simulate any Turing machine, with neuron dimension and generation
% length polynomial in the machine's running time and
% input length (\Cref{cor:linear-threshold-universal}).
% In this Boolean setting, our simulation improves
% on \citet{joshi2025theory} by replacing their exponential
% dependence on circuit depth with depth-independent
% polynomial bounds.
For threshold activations, our simulation improves
on \citet{joshi2025theory} by replacing their exponential
dependence on circuit depth with depth-independent
polynomial bounds. It implies that a single time-invariant
AR neuron can simulate any Turing machine, with neuron dimension and generation length polynomial in the machine's running time and
input length (\Cref{cor:linear-threshold-universal}).
Our result also captures real-valued computation for a broad
class of nonlinear activations: a single shared neuron can
efficiently simulate any size-$s$ feedforward network with
the same activation on $[0,1]^n$ (\Cref{thm:LT-expressivity-tight}).

\item \textbf{Learnability (\Cref{sec:cot-learning}).}
Combining our simulation with known learning guarantees
for AR linear thresholds shows that a single
autoregressive threshold neuron can be efficiently
learned from complete CoT trajectories to predict
the outputs of any polynomial-time Turing machine (\Cref{cor:lt-are-universal-learner}). 
For ReLU, combining our simulation with known
generalization bounds for neural networks
 yields analogous learning
guarantees for feedforward ReLU networks (\Cref{cor:relu-cot-learning-informal}).
% Together, these results answer \Cref{ques:leanability}.
%In contrast, 
Another implication of our result is that, in the absence of CoT, learning from final answers alone is hard, assuming only the existence of one-way functions. 
Prior work established such hardness under stronger cryptographic assumptions \citep{joshi2025theory,joshi2026learning}.
% Our result also implies that
% learning from final answers alone is computationally hard
% under the standard cryptographic assumption that one-way
% functions exist (\Cref{cor:e2e-hardness-informal}).
% This assumption is much weaker than those required in previous work \citep{joshi2025theory,joshi2026learning}. 

\item \textbf{Optimality of the simulation bounds
(\Cref{sec:lower-bounds}).}
The dimension bound $d=O(ns+s^2)$ matches the worst-case
number of circuit weights. For generation length, we prove a matching
worst-case lower bound of $T=\Omega(s^2)$
(\Cref{thm:direct-alignment-lower-bound}) under a natural structural
assumption shared by our construction and prior simulations.

    \item \textbf{Experiments (\Cref{sec:experiments-main}).}
    We use a CNN teacher with nontrivial test accuracy on CIFAR-10 \citep{krizhevsky2009learning} and train a single shared autoregressive ReLU neuron on its activation traces, without access to the teacher's weights.
    The learned AR neuron agrees with the teacher's predictions on approximately $99\%$ of test examples.

\end{itemize}

The proof of \Cref{thm:main-informal} is based on a novel connection between autoregressive circuit simulation
and Sidon sets from additive combinatorics. Using the classical construction
of \citet{erdos1941problem}, we position the circuit weights according to a
Sidon set within the parameter vector of a single neuron. This careful
placement prevents weights from colliding during the autoregressive process,
despite parameter sharing. The resulting collision-free encoding is the key
new ingredient, yielding a depth-independent simulation.
% %, and offering a simpler yet
% substantially more powerful alternative to the recursive construction of
% \citet{joshi2025theory}.

\section{Preliminaries}
\paragraph{Notation.}
We use boldface letters for vectors and write $[n]:=\{1,2,\dots,n\}$. For a
vector $\vv\in\R^m$ and $i\in[m]$, we write $\vv[i]$ for its $i$-th coordinate
and $\vv[-i]:=\vv[m-i+1]$ for its $i$-th coordinate from the right. For real
numbers $a,b$, let $a\lor b:=\max\{a,b\}$ and $a\land b:=\min\{a,b\}$.
We use standard asymptotic notation, with $\Theta(\cdot)$, $\Omega(\cdot)$,
and $O(\cdot)$ having universal implicit constants, and $\Tilde{\Theta}(\cdot)$, $\Tilde{\Omega}(\cdot)$, and $\Tilde{O}(\cdot)$
additionally suppress factors that are polylogarithmic in the problem
parameters. We write $\ind[P]$ for the indicator of a predicate $P$.
 \newline \newline
Next, we define the autoregressive single-neuron model, the circuits
it simulates, and the learning setting with complete CoT supervision.
\paragraph{Time-invariant Autoregressive Generation.}
Let $\Sigma$ be a token alphabet and let
$\cF\subseteq\Sigma^{\Sigma^*}$ be a class of next-token predictors,
called the \emph{base class}.
We use $\Sigma=\{0,1\}$ for Boolean computation and $\Sigma=\R$ for real-valued computation.
For an input $\vx\in\Sigma^n$, starting from the initial context
$\vx^{(0)}=\vx$, autoregressive generation repeatedly applies the
same predictor $f\in\cF$ and appends its output to the context.
We write $\ete{f}{t}(\vx)$ for the output token generated at the
$t$-th autoregressive step:
\begin{align*}
    \ete{f}{t}(\vx)&=f(\vx^{(t-1)}),&
    \vx^{(t)}&=(\vx^{(t-1)},\ete{f}{t}(\vx)).
\end{align*}
We denote the resulting chain of thought by
$\CoT{f}{T}(\vx)=(\ete{f}{1}(\vx),\ldots,\ete{f}{T}(\vx))\in\Sigma^T$,
consisting of the $T$ generated tokens.
Reusing the same predictor at every step makes the model \emph{time-invariant}.

% For $\vx\in\R^n$ and $\vw\in\R^d$, we interpret their inner product by
% right-aligning their coordinates:
% \begin{align*}
%     \langle \vw,\vx\rangle
%     := \sum\nolimits_{i=1}^{d\land n}\vw[-i]\vx[-i].
% \end{align*}
% Equivalently, this is the standard inner product obtained by padding the
% shorter vector with zeros on the left. This right-alignment convention is natural for autoregressive models, where the last $d$ values of the input are used to generate the next value.
% \subsection{Autoregressive \texorpdfstring{$f$}{f}-linear models.}
\paragraph{Single-neuron predictors.} Our base class consists of single neurons with a fixed activation
$\sigma:\R\to\Sigma$, where $\Sigma\subseteq\R$.
For $\vx\in\Sigma^n$ and $\vw\in\R^d$,
we use the right-aligned inner product
$\langle\vw,\vx\rangle:=\sum_{i=1}^{d\land n}\vw[-i]\vx[-i]$, equivalently
padding the shorter vector with zeros on the left. This convention aligns each
new CoT value with the shared parameter vector. For $b\in\R$, define the next-token predictor
$f_{\vw,b,\sigma}:\Sigma^*\to\Sigma$ by
$f_{\vw,b,\sigma}(\vx):=\sigma(\langle\vw,\vx\rangle+b)$.
The resulting dimension-$d$ base class is
\begin{align*}
    \cF_{d,\lin,\sigma}
    :=\left\{f_{\vw,b,\sigma}:\vw\in\R^d,\ b\in\R\right\}.
\end{align*}
Thus, $\cF_{d,\lin,\sigma}$ is a class of single neurons with
activation $\sigma$, or equivalently generalized linear models
with link function $\sigma$. 
When $b=0$, we write $f_{\vw,\sigma}:=f_{\vw,0,\sigma}$.
When $\sigma$ is clear from context, we omit it from the notation
and write $\cFlin$, $f_{\vw,b}$, and $f_{\vw}$.

\paragraph{Efficient Learning with CoT.} We study learning these AR predictors from complete
CoT supervision \citep{joshi2025theory}. Let $\cF_d$ be a base class of next-token predictors, where $d$ measures
predictor size. Given target generator $f_\star \in \cF_d$, each training example reveals
the input $\vx$ and the entire trajectory $\CoT{f_\star}{T}(\vx)$. 
The learner is evaluated on the final output $\ete{f_\star}{T}(\vx)$.\footnote{Our positive results
also hold under the stronger evaluation criterion that requires
the learner to match the entire CoT trajectory.}
Here, we define CoT learnability in the Boolean setting
$\Sigma=\{0,1\}$. Learning setting for real-valued  models are given in \Cref{sec:learning-relu-NN-with-cot}.

\begin{definition}[Realizable CoT learnability]\label{def:cot-learnability}
We say that $\ete{\cF_d}{T}$ is CoT-learnable over $\{0,1\}^n$ with generation
length $T$ and sample complexity $m(\varepsilon,\delta)$ if there is a
learning rule $A$ such that the following holds.  For every distribution
$\cD$ over $\{0,1\}^n$, target generator $f_\star\in\cF_d$, and
$\varepsilon,\delta\in(0,1)$, let
\begin{align*}
S_{\mathrm{CoT}}
    =\bigl((\vx_1,\CoT{f_\star}{T}(\vx_1)),\ldots,
           (\vx_m,\CoT{f_\star}{T}(\vx_m))\bigr),
\qquad
\vx_1,\ldots,\vx_m\stackrel{\mathrm{i.i.d.}}{\sim}\cD,
\end{align*}
where $m=m(\varepsilon,\delta)$.  Then, with probability at least
$1-\delta$ over the sample and the randomness of $A$, the predictor
$h=A(S_{\mathrm{CoT}}):\{0,1\}^n\to\{0,1\}$ satisfies
\begin{align*}
\Pr_{\vx\sim\cD}\left[h(\vx)\neq\ete{f_\star}{T}(\vx)\right]
    \leq\varepsilon.
\end{align*}
\end{definition}
\paragraph{The Computational Aspect.} We call $\ete{\cF_d}{T}$ efficiently
CoT-learnable if there exists a learning rule $A$ satisfying
\Cref{def:cot-learnability} whose sample complexity and running time are
polynomial in $n,d,T,1/\varepsilon$, and $\log(1/\delta)$.

\paragraph{$\sigma$-circuits.}
We next define the circuits that serve as simulation targets.
A size-$s$, depth-$L$ $\sigma$-circuit
$C:\Sigma^n\to\Sigma$ is a weighted directed acyclic graph
$G=(V,E)$ with input nodes
$V_{\text{in}}=\{v_{\text{in},1},\dots,v_{\text{in},n}\}$
at depth $0$ and non-input nodes $\{v_1,\dots,v_s\}$.
The non-input nodes are partitioned into layers
$V^{(1)},\dots,V^{(L-1)},\{v_s\}$, where $V^{(\ell)}$
contains the hidden nodes at depth $\ell$ and $v_s$ is the
output node at depth $L$.
Edges point from smaller to larger depth, with weights
$w:E\to\R$.

For input $\vx=(x_1,\dots,x_n)\in\Sigma^n$, each node
$v\in V$ has a value $\val(v)\in\Sigma$.
Input nodes have values $\val(v_{\text{in},k})=x_k$,
and each non-input node $v_i$ has a bias $\theta_i\in\R$
and computes
\begin{align*}
    \val(v_i)
    =
    \sigma\left(\theta_i+\sum_{(u,v_i)\in E}w(u,v_i)\,\val(u)\right).
\end{align*}
The output is $C(\vx)=\val(v_s)$, and we write
$C_i:\Sigma^n\to\Sigma$ for the function
$C_i(\vx):=\val(v_i)$ computed at node $v_i$.
The circuit's \emph{size} is its number of non-input nodes,
$|V\setminus V_{\text{in}}|=s$.
We call the circuit \emph{layered} if every edge connects
consecutive layers.

For threshold activation $\sigma(z)=\mathbf{1}\{z\geq 0\}$ and $\Sigma=\{0,1\}$, these are \emph{threshold circuits} \citep{arora2009computational}.
For real-valued activations, these circuits are feedforward neural
networks with activation $\sigma$, including MLPs and CNNs with
or without skip connections.
In this setting, we also consider inputs restricted
to $[0,1]^n$, while retaining node and output values in $\R$.

\section{Expressivity of Autoregressive Single Neuron}\label{sec:expressivity}

In this section, we present our main result: for a broad class of activations $\sigma$,
a single autoregressive neuron can simulate a $\sigma$-circuit
using the same activation. For size-$s$ circuits on $n$ inputs, our simulations use
dimension $d=O(ns+s^2)$ and generation length $T=O(s^2)$,
after applying an input feature map that depends only on $n$ and $s$ (see \Cref{remark:common-feature-map}).
We begin in \Cref{subsec:boolean} with the discrete token space $\Sigma=\{0,1\}$ and threshold activations, 
showing how to simulate boolean threshold circuits and Turing machines. 
In \Cref{subsec:real}, we consider general activations with real-valued computations, 
showing that a single neuron can simulate general feed-forward neural networks with the same activation. 
Finally, \Cref{subsec:proof-technique} outlines how Sidon sets enable efficient serialization of circuit computations.

\subsection{Autoregressive Linear Thresholds Simulate Boolean Threshold Circuits}\label{subsec:boolean}
We begin with the threshold activation $\sigma(z)=\mathbf{1}\{z\geq0\}$, which yields the autoregressive linear-threshold class studied in prior work \citep{joshi2025theory,joshi2026learning,altabaa2025cot,hanneke2026sample,doron2026theory}.
% We show that, up to a fixed polynomial-time computable feature map of the input, every size-$s$ Boolean threshold circuit can be represented by an autoregressive linear threshold of dimension $d=O(ns+s^2)$ and generation length $T=O(s^2)$. 
This is arguably the simplest autoregressive model over the binary token space $\Sigma=\{0,1\}$: taking the logits to be the linear scores $(\langle \vw,\vx\rangle+b)/2$ and $-(\langle \vw,\vx\rangle+b)/2$ for the two tokens $1$ and $0$, respectively, yields a logistic next-token distribution after applying Softmax. Then, greedy decoding is exactly the linear threshold rule $\mathbf{1}\{\langle \vw,\vx\rangle+b\geq 0\}$. Below is our main expressivity result, stated for the binary token space $\Sigma=\{0,1\}$:

\begin{theorem}\label{thm:boolean-LT-expressivity} Consider $\sigma(z)=\mathbf{1}\{z\geq 0\}$, and the class $\cF_{d,\lin,\sigma}$ over the token space $\Sigma=\{0,1\}$. For every input size $n$ and circuit size $s$, there is a map $\phi_{n,s}:\{0,1\}^n\to\{0,1\}^{n'}$ for $n'=O(ns+s^2)$, with the following property. Every size-$s$ Boolean threshold circuit $C:\{0,1\}^n\to\{0,1\}$ can be simulated by a single autoregressive linear threshold neuron $f_{\vw,\sigma}$ with $\vw\in\R^d$, where $d=O(ns+s^2)$, in $T=O(s^2)$ generation steps:
\[
\ete{f_{\vw,\sigma}}{T}(\phi_{n,s}(\vx))
=
C(\vx)
\qquad
\text{for every $\vx\in\{0,1\}^n$}.
\]
% Equivalently, for $d=O(ns+s^2)$ and $T=O(s^2)$, we have \( \mathrm{TC}_n(s)\subseteq \ete{\cF_{d,\lin,\sigma}}{T}\circ \phi_{n,s}. \) \roey{$\mathrm{TC}_n(s)$ is not defined.}

Moreover, $\phi_{n,s}$ depends only on $n$ and $s$ and is computable in time polynomial in $n$ and $s$.
\end{theorem}

\begin{remark}[A common feature map]\label{remark:common-feature-map}
For any fixed input size $n$ and circuit-size bound $s$, the map
$\phi_{n,s}$ applies simultaneously to all threshold circuits of size at most
$s$; only the parameter vector $\vw$ depends on the circuit $C$.
Indeed, smaller circuits can be padded to size $s$ by inserting
dummy gates that output zero before the output gate, without changing
the computed function. Such feature maps also appear in prior simulation results for
specific autoregressive neurons \citep{joshi2025theory,joshi2026learning,doron2026stochastic}. 
\end{remark}

Combining \Cref{thm:boolean-LT-expressivity} with the standard simulation of polynomial-time Turing machines by polynomial-size threshold circuits \citep{pippenger1979relations,arora2009computational}, 
yields the following universality result:

\begin{corollary}[Turing machine simulation]\label{cor:linear-threshold-universal}
Let $\tau(n)=\poly(n)$ be a running-time bound.
There exists a family of $\poly(n)$-time computable feature maps
$\phi_n:\{0,1\}^n\to\{0,1\}^{n'}$, where
$n'=\tilde{O}(n\tau(n)+\tau(n)^2)$, such that the following holds.
For every deterministic Turing machine $M$ running in time at most
$\tau(n)$ and every sufficiently large $n$,\footnote{The required lower bound on $n$ may depend on $M$,
but the feature maps depend only on $n$ and $\tau$.
An extra logarithmic factor in the common circuit-size bound absorbs
machine-dependent simulation constants; see
\Cref{remark:common-feature-map} for the padding argument.}
there exists a linear threshold $f_{\vw}$ with $\vw\in\R^d$ such that
\begin{align*}
\ete{f_{\vw}}{T}(\phi(\vx))=M(\vx)
\qquad\text{for every $\vx\in\{0,1\}^n$},
\end{align*}
where $d=\tilde{O}(n\tau(n)+\tau(n)^2)$ and
$T=\tilde{O}(\tau(n)^2)$.
\end{corollary}

Moreover, \Cref{thm:boolean-LT-expressivity} replaces the bounds
$d,T=O(s^Ln)$ of \citet[Lemma~4.5]{joshi2025theory}, where $L$ is the circuit depth, with
$d=O(ns+s^2)$ and $T=O(s^2)$.
This removes the exponential dependence on circuit depth,
enabling efficient simulation of polynomial-time Turing machines.
% In this setting, \Cref{thm:boolean-LT-expressivity} also strengthens the
% simulation result of \citet{joshi2025theory}. Their construction
% (Lemma~4.5 therein) simulates a threshold circuit of size $s$ and depth $L$
% over $\{0,1\}^n$ using an autoregressive linear-threshold model whose dimension $d$
% and number of steps $T$ are both $O(s^Ln)$. The depth of the circuit simulating a
% Turing machine scales linearly with its running time, so these bounds incur an
% exponential overhead in the running time. In sharp contrast, our construction has
% dimension $d=O(ns+s^2)$ and uses $T=O(s^2)$ autoregressive steps, both completely
% independent of the circuit depth. This independence removes the exponential
% dependence on the running time and yields only polynomial overhead. \roey{we can remove this part if we want}

\subsection{Neural Networks can be simulated by a Single Autoregressive Neuron}\label{subsec:real}

More broadly, our expressivity result is not specific to threshold activations: it applies to a broad class of standard nonlinear activations. 
Here we take $\Sigma=\R$. In this case, $\sigma$-circuits correspond to standard feedforward neural networks.
%, which are widely used in practice and have been studied extensively in theoretical work on expressivity and learning.  \nirmit{Recall what autoregressive setup for $\sigma$ would mean and that $\Sigma=\R$. Just a line. You can add a line or two on why this setup is interesting and natural.}  
To state the results formally, we first specify the assumptions on the activation $\sigma$.
\begin{assumption}\label{ass:activation}
There exists $c_\sigma\in\R$ such that $\sigma(z)=0$ for every $z\leq c_\sigma$ (the zero-region condition). Moreover, $\sigma$ is bounded on every bounded interval.
\end{assumption}

The zero-region condition is satisfied by activations such as threshold and ReLU. Many other common activations, including sigmoid, GELU, and SiLU, do not become exactly zero below any finite threshold, but instead converge to zero as their input tends to $-\infty$. For such activations, we obtain an approximate version of our expressivity result. In particular, the zero-region condition in \Cref{ass:activation} can be replaced by a decay to zero at $-\infty$, together with a mild regularity condition.

\begin{assumption}\label{ass:activation-approximate}
The activation $\sigma:\R\to\R$ is locally Lipschitz (\Cref{def:locally-lipschitz}) and satisfies $\lim_{z\to-\infty}\sigma(z)=0$.
\end{assumption}

Below, we state our exact and approximate expressivity result, under \Cref{ass:activation,ass:activation-approximate} respectively, for simulating general feed-forward neural networks with the same activation and over a bounded domain.

\begin{theorem}\label{thm:LT-expressivity-tight}
For every input dimension $n$ and circuit size $s$, there exist
a feature map $\phi_{n,s}:[0,1]^n\to[0,1]^{n'}$ and
integers $d,T$, with $n',d=O(ns+s^2)$ and $T=O(s^2)$,
such that the following holds for every size-$s$
$\sigma$-circuit $C:[0,1]^n\to\R$:
\begin{enumerate}
    \item \textbf{Exact simulation.}
    If $\sigma$ satisfies \Cref{ass:activation}, there
    exists $\vw\in\R^d$ such that
    \begin{align*}
    \ete{f_{\vw,\sigma}}{T}(\phi_{n,s}(\vx))=C(\vx)
    \qquad\text{for every $\vx\in[0,1]^n$}.
    \end{align*}

    \item \textbf{Approximate simulation.}
    If $\sigma$ satisfies
    \Cref{ass:activation-approximate}, then for every
    $\varepsilon>0$, there exists $\vw\in\R^d$ such that
    \begin{align*}
    \left|
    \ete{f_{\vw,\sigma}}{T}(\phi_{n,s}(\vx))-C(\vx)
    \right|\leq\varepsilon
    \qquad\text{for every $\vx\in[0,1]^n$}.
    \end{align*}
\end{enumerate}
Moreover, $\phi_{n,s}$ depends only on $n$ and $s$ and is
computable in time polynomial in $n$ and $s$.\footnote{Here,
we use the real-RAM model, in which reading, writing, and
copying a real number take unit time.}
\end{theorem}

Our exact simulations also preserve the circuit's intermediate
gate values in the generated trajectory, as formalized by
the following property.

\begin{property}[CoT preservation]
\label{property:cot-preservation}
The exact simulation in \Cref{thm:LT-expressivity-tight}
can be chosen so that there exist designated steps
$1\leq t_1<\cdots<t_s=T$, depending only on $n,s$, such that
\begin{align*}
\ete{f_{\vw,\sigma}}{t}(\phi_{n,s}(\vx))
=
\begin{cases}
    C_i(\vx), & \text{if } t=t_i \text{ for some } i\in[s],\\
    0, & \text{otherwise},
\end{cases}
\qquad
\text{for every $\vx\in[0,1]^n$ and $t\in[T]$}.
\end{align*}
The designated steps are computable in time polynomial
in $n,s$. The same property holds for the Boolean simulation
in \Cref{thm:boolean-LT-expressivity}, with inputs restricted
to $\{0,1\}^n$.
\end{property}

\subsection{Proof technique: serialization via Sidon sets}\label{subsec:proof-technique}
The main technical contribution is a new
%way of viewing the simulation: 
simulation approach:
we
serialize the circuit computation into an autoregressive sequence and use
an additive-combinatorial structure to prevent different circuit weights from
colliding within a single shared parameter vector. 

Each non-input circuit node $v_i$, for $i\in[s]$, is assigned a generation
step $t_i\in [T]$, at which the value $C_i(\vx)$ is produced, i.e. $ \ete{f_{\vw,\sigma}}{t_i}(\phi_{n,s}(\vx))=C_i(\vx)$. At any later step
$t_j$, where $j>i$, this value is the $(t_j-t_i)$-th most recent value in the
autoregressive context and is therefore aligned with the coordinate
$\vw[-(t_j-t_i)]$, where $\vw$ is the shared parameter vector. To simulate a circuit edge $(v_i,v_j)\in E$ with weight $w(v_i,v_j)$, we therefore set
$\vw[-(t_j-t_i)] = w(v_i,v_j)$, so that $C_i(\vx)$ contributes
$w(v_i,v_j)C_i(\vx)$ to the preactivation at step $t_j$, matching the corresponding term in the definition of a $\sigma$-circuit.

The difficulty is that each coordinate of $\vw$ is reused at every generation
step. Consequently, if $t_j-t_i=t_\ell-t_k$, for two distinct pairs $(i,j)$ and $(k,\ell)$, then two potentially different
circuit weights would have to be encoded by the same coordinate of $\vw$. Similarly, aligning input-to-gate weights requires distinct sums
$t_i+t'_q$, where $t'_q$ denotes the initial position assigned to the
corresponding input value, counted from the right. Sidon sets provide the required collision-free assignment.
Formally, a set
$S\subseteq\mathbb{N}$ is a \emph{Sidon set} if
\[
    a+b=c+d
    \quad\Longrightarrow\quad
    \{a,b\}=\{c,d\}
    \qquad\text{for all }a,b,c,d\in S.
\]
Equivalently, its positive pairwise \emph{differences} are also distinct (\Cref{obs:sidon.set.eq.golomb.ruler}). Using the classical construction of \citet{erdos1941problem}, we obtain such a
set of size $s$ whose largest element is $O(s^2)$. The precise formulation we use is stated and proved in \Cref{claim:explicit_sidon_set}. 
Then we use this set to choose the generation steps $t_1,\dots,t_s$ and to construct the feature map $\phi_{n,s}$.
The formal proof completes this picture by verifying that the values generated at undesignated time
steps do not interfere with the simulation. The main novelty lies in viewing the proof through this collision-free
serialization framework.

\section{Implications for Learning}\label{sec:cot-learning}
\paragraph{Efficient Learning with CoT.} We have shown that a single AR neuron can efficiently simulate
general computations such as Turing machines and feedforward neural networks.
The $\sCoT$ in our simulations is the neuron's complete autoregressive
trajectory: the intermediate tokens perform the simulated computation (by \Cref{property:cot-preservation}), and the final token is the output of the
simulated machine or network. Suppose that these complete trajectories are provided during training.
For a fixed activation $\sigma$, each trajectory generated by a
target neuron $f_\star\in\cF_{d,\lin,\sigma}$ provides labeled
prefix--next-token examples for the same shared predictor.
Consequently, it suffices to find a neuron in
$\cF_{d,\lin,\sigma}$ that predicts the labeled next token for
each prefix exactly or approximately.
For some activation functions, such a neuron can be found
efficiently. Combined with an appropriate generalization bound,
this computational guarantee yields efficient learnability from CoT supervision (\Cref{def:cot-learnability}). \newline

% Together with the efficient feature map, it allows us
% to convert observed circuit activations into complete
% autoregressive trajectories without knowing the circuit
% weights. Suppose that these complete trajectories are provided during training. Each trajectory then supplies labeled
% prefix--next-token examples for the same shared neuron.
% Consequently, it suffices to find a neuron in
% $\cF_{d,\lin,\sigma}$ that predicts the labeled next token for
% each prefix exactly or approximately.
% For some activation functions, such a neuron can be found
% efficiently. Combined with an appropriate generalization bound,
% this computational guarantee yields efficient learnability.

% \paragraph{Sample Complexity \& Learning.}
% \citet{joshi2025theory} distinguish two ways of learning an
% autoregressive model. Under CoT supervision, each training example contains
% the complete trajectory $\CoT{f}{T}(\vx)$, whereas under end-to-end (E2E)
% supervision it contains only the final token $\ete{f}{T}(\vx)$. In both
% settings, the learner is evaluated on its prediction of the final output.

For the threshold activation, linear programming can efficiently
find parameters consistent with all observed prefix--next-token
pairs. Combining \Cref{cor:linear-threshold-universal} with known generalization bounds for autoregressive  linear thresholds (e.g.\Cref{thm:LT-learnable-with-COT}) yields the following:

\begin{corollary}\label{cor:lt-are-universal-learner}
Let $M$ be a polynomial-time Turing machine, and let $\cD$ be
any distribution over $\{0,1\}^n$. Given a sufficiently large polynomial number of complete CoT
trajectories of length $T=\poly(n)$ encoding the computations of $M$ on i.i.d.\ inputs drawn from
$\cD$, one can efficiently train a single linear-threshold
neuron $f_{\vw}$ to predict $M(\vx)$ with high probability over a fresh $\vx \sim \cD$. 
% $\ete{f_{\vw}}{T}(\vx)\simeq M(\vx)$\footnote{Formally,
% $\ete{f_{\vw}}{T}(\psi(\vx))=M(\vx)$, where $\psi$ is from
% \Cref{cor:linear-threshold-universal} and is the same for all Turing machines
% with the same running-time bound as $M$.}
%\gal{Is this notation introduced?}
% with high probability over a fresh $\vx \sim \cD$.
\end{corollary}

Similarly, a single AR ReLU neuron can simulate an entire, potentially large and deep,
feedforward ReLU network (\Cref{thm:LT-expressivity-tight}). With complete activation supervision, fitting this shared neuron
reduces to linear programming as well.\footnote{For a fixed prefix $\vz$
and target activation $y$, consistency here requires
$\langle\vw,\vz\rangle+b=y$ when $y>0$ and
$\langle\vw,\vz\rangle+b\leq0$ when $y=0$.} 
Combined with generalization bounds for ReLU Neural Networks \citep{bartlett2019nearly} yields the following:

\begin{corollary}[Informal version of \Cref{cor:relu-cot-learning,remark:regression}]\label{cor:relu-cot-learning-informal}
For any distribution over $[0,1]^n$, complete activation
supervision from a size-$s$ ReLU network $C$ suffices to train
a single autoregressive ReLU neuron to predict $\ind[C(\vx)>0]$ (or approximate $C(\vx)$ under a suitable bounded regression loss) with high probability on a fresh input.
The neuron's dimension, generation length, sample complexity,
and training time are polynomial in the relevant parameters.
\end{corollary}

% Gradient-based methods offer an alternative approach to
% fitting the supervised trajectories.
% For threshold predictors, SGD on a suitable convex surrogate
% loss admits margin-dependent learning guarantees
% \citep{shalev2014understanding}.
% For ReLU predictors, complete activation supervision makes
% the loss (see \Cref{eq:loss}) convex and smooth, allowing
% approximate minimization by gradient descent.

In both cases, gradient-based methods offer an alternative to linear programming for fitting the observed prefix--next-token pairs. Penalizing squared violations of the linear feasibility constraints yields a convex, smooth loss whose minimum value is zero (under realizability); for ReLU, this gives the per-example loss in \Cref{eq:loss} used
in our experiments. Gradient descent with a suitable step size can therefore make this
loss arbitrarily small.

% \paragraph{Consistency under complete CoT supervision.}
% The \sCoT-preserving simulation also yields a simple algorithmic observation:
% fitting the shared autoregressive ReLU neuron to complete CoT trajectories is
% a linear-program feasibility problem. Indeed, each observed
% prefix is fixed: a positive next-token value requires the corresponding
% preactivation to equal that value, whereas a zero next-token value requires
% the preactivation to be nonpositive. The \sCoT-preserving property therefore
% turns the hidden activations of the simulated feedforward network into labeled
% next-token examples for the same shared neuron. This observation concerns
% complete CoT supervision and does not imply tractable learning from E2E
% examples alone. \roey{Need to emphasize that COT from NN are not observed in real life}

% This parameter dependence also highlights the near-optimality of our
% simulation. The VC dimension, and hence the sample complexity, of threshold
% circuits scales with their number of trainable parameters (number of edges) up to
% polylogarithmic factors \citep{AnBa02}. A size-$s$ circuit with $n$ input nodes have at most $\Theta(ns+s^2)$ edges (weights), so the number of circuit parameters is
% $\Theta(ns+s^2)$, while our autoregressive simulation uses
% $d=O(ns+s^2)$. Thus, our simulation is
% essentially tight in terms of sample complexity, up to polylogarithmic
% factors.

\paragraph{Hardness of E2E Learning.}
We return to linear-threshold models with $\Sigma=\{0,1\}$ and
consider learning without access to the intermediate trajectory.
Under End-to-End (E2E) supervision, a training example reveals
only the input $\vx$ and the final output
$\ete{f_\star}{T}(\vx)$ of a target generator $f_\star\in\cFlin$.
The learner is evaluated on the same final output, but,
unlike under CoT supervision, all intermediate values remain latent during training.
Formal definitions are given in \Cref{sec:learning-with-cot}.

\citet{joshi2025theory} showed that E2E learning of autoregressive linear thresholds is
computationally hard under a hardness assumption on learning constant-depth
threshold circuits.  Their reduction simulates a size-$s$, depth-$L$ threshold
circuit using dimension and generation length of order $s^L$.  Thus, the
reduction has polynomial overhead only for constant $L$, which limits the
computational assumptions from which it can derive hardness.

Our simulation removes this limitation.  Every polynomial-time computation
is simulated with polynomial dimension and generation length. Combining this simulation with the existence of one-way
functions yields E2E hardness under one of the most common cryptographic
assumptions.  

\begin{corollary}[Informal version of \Cref{cor:e2e_hardness_from_owf}]\label{cor:e2e-hardness-informal}
If one-way functions exist, then no polynomial-time algorithm learns 
autoregressive linear-threshold with E2E supervision.  
\end{corollary}

Indeed, the existence of one-way functions can be viewed as a minimal assumption for computational cryptography,
already sufficient for fundamental primitives such as pseudorandom generators, pseudorandom functions, private-key encryption and digital signatures \citep{goldreich2007foundations_set}.

\section{Optimality of the simulation bounds on \texorpdfstring{$T$}{T} and \texorpdfstring{$d$}{d}}
\label{sec:lower-bounds}

In this section, we discuss the optimality of the bounds of the generation length $T$ and dimension $d$ in \Cref{thm:LT-expressivity-tight}.
We show that the bounds are tight up to constant factors, under some reasonable assumptions on the construction. 
\citet{joshi2026learning} showed simulation of depth-2 threshold circuit with $T=O(s)$ and $d=O(ns)$. 
This motivates the question of whether we can improve the bounds when the circuit has additional structure, such as bounded depth or width.
We show that even then, we cannot expect to improve those bounds in some cases.

\paragraph{Dimension d.}
Since the simulation should be able to represent every circuit, it is reasonable to expect that all the weights of the circuit should be encoded in the autoregressive model, that is in the parameter vector $\vw$.\footnote{This assumption does not yield an unconditional lower bound, since edge counting alone does not preclude the existence of a more compact simulation. However, any such simulation must apply to arbitrary circuits, including those that are already "compressed".}
A size-$s$ circuit has at most $ns$ edges from input nodes and at most
$\binom{s}{2}$ edges between non-input nodes. Therefore, the bound of $d=O(ns+s^2)$ matches the number of circuit weights in the worst case.
Even for constant depth circuits, the number of edges between non-input nodes can be $\Theta(s^2)$. Indeed, a depth-$3$ circuit can have two consecutive
layers of $\Theta(s)$ non-input nodes with all edges between them, and hence
$\Theta(s^2)$ internal edges.

\paragraph{Generation length T.}
We next show that a quadratic generation length is necessary under a natural restriction on how the model simulates the circuit:
Each non-input node value $C_i(\vx)$ must be generated at a designated step (as in \Cref{property:cot-preservation}), and
for every edge between non-input nodes, the model must use the generated
value of the source node directly when computing the target node,
with a coefficient equal to the edge weight.
We formalize this restriction below.

\begin{definition}[Direct edge alignment]\label{def:direct-edge-alignment}
Fix $\Sigma=\R$ and an activation $\sigma:\R\to\R$.
Consider a simulation scheme that simulates every $\sigma$-circuit
$C$ by an autoregressive neuron $f_{\vw,\sigma}$ using any feature map $\phi$.
The scheme is \emph{directly edge-aligned} if, for every circuit
with non-input nodes $v_1,\dots,v_s$, there are designated steps
$1\leq t_1<\cdots<t_s\leq T$ such that:
\begin{enumerate}
    \item $\ete{f_{\vw,\sigma}}{t_i}(\phi(\vx))=C_i(\vx)$
    for every $i\in[s]$ and $\vx\in[0,1]^n$.
    \item For every edge $(v_j,v_i)$ between non-input nodes, the value
    generated at step $t_j$, namely $C_j(\vx)$, is used at step $t_i$
    with coefficient $w(v_j,v_i)$.\footnote{Formally, the value $C_j(\vx)$ generated at step
$t_j$ is the $(t_i-t_j)$-th most recent value in the autoregressive input.
Thus, being used with coefficient $w(v_j,v_i)$ means that
$\vw[-(t_i-t_j)]=w(v_j,v_i)$, so this value contributes
$w(v_j,v_i)C_j(\vx)$ to the preactivation at step $t_i$.}
\end{enumerate}
\end{definition}
For threshold circuits on Boolean inputs, the same definition
applies with $\Sigma=\{0,1\}$ and $\vx\in\{0,1\}^n$. The construction in the proof of \Cref{thm:LT-expressivity-tight} is
directly edge-aligned, as well as all the previous constructions
\citep{malach2023auto, joshi2025theory,joshi2026learning}. Direct edge alignment is a natural restriction, since it allows the autoregressive model to simulate the circuit in a straightforward manner. Under this restriction, the
quadratic dependence on the circuit depth is necessary, even for width-$1$ circuits:

\begin{theorem}\label{thm:direct-alignment-lower-bound}
For every $L\geq1$, any directly edge-aligned simulation
(\Cref{def:direct-edge-alignment}) for arbitrary $\sigma$-circuits
has generation length $T=\Omega(L^2)$ on depth-$L$ circuits,
even when restricted to width-$1$ circuits.
\end{theorem}

In the worst case, a circuit with $s$ nodes may have depth proportional to $s$,
leading to a generation length of $T=\Omega(s^2)$. 
This holds even for sparse circuits whose non-input nodes form a single directed path.

% \paragraph{Open question.} Can additional circuit structure yield non-trivial bounds that depend on the number of edges rather than on $s^2$? For example, consider a layered circuit of depth $\sqrt{s}$ and width $\sqrt{s}$. 
% Such a circuit might have $\Theta(s^{3/2})$ edges between non-input nodes and at most $n\sqrt{s}$ edges from the input nodes. 
% Is there a simulation for this class with $T=O(s^{3/2})$ and $d=O(s^{3/2}+n\sqrt{s})$?

\section{Experiments}\label{sec:experiments-main}
We test whether a single shared neuron trained efficiently from complete activation traces
generalizes on a nontrivial vision task.  We first train a four-hidden-layer
ReLU CNN teacher $C$ on CIFAR-10~\citep{krizhevsky2009learning},
obtaining nontrivial test accuracy.
Given an image $\vx$, we record its
intermediate activations $C_1(\vx), C_2(\vx), \dots$ and ten class logits.
We serialize these values using the construction in
\Cref{thm:LT-expressivity-tight} with the affine-output extension in
\Cref{remark:affine.readout}: the student uses shared weights $\vw$, ReLU at
hidden and undesignated steps, and identity activation at ten designated
output steps. For efficiency, we represent the AR neuron (i.e., $\vw$) using a sparse implementation 
which enables very efficient training and inference. We independently initialize the trainable parameters underlying $\vw$
and optimize them using Adam. Importantly, the student is trained only on activation traces and has no
access to the teacher's weights. Training the AR neuron on $4{,}000$ teacher activation trajectories takes roughly $5$ minutes on an NVIDIA L40S GPU.

At inference time, we classify images by evaluating the AR neuron only at designated steps. 
The obtained test accuracy on CIFAR-10 is almost identical to the teacher's; \Cref{tab:headline} summarizes the results. 
This experiment illustrates the mechanism
behind the theory: complete CoT traces transform a general computation into
supervised examples, enabling efficient training of a single neuron. See \Cref{sec:experiments} for further details.

\begin{table}[h]
    \centering
    \caption{CIFAR-10 test results. Student results show mean $\pm$ standard
deviation.}
    \label{tab:headline}
    \small
    \begin{tabular}{@{}lccc@{}}
        \toprule
        Model & Supervision & Test accuracy & Teacher agreement \\
        \midrule
        CNN teacher & 45,000 images & $72.150\%$ & $--$ \\
        AR student & 4,000 traces
            & $(72.106 \pm 0.054)\%$ & $(99.484 \pm 0.070)\%$ \\
        \bottomrule
    \end{tabular}
\end{table}

\section{Conclusion and Future Directions}
We show that, perhaps surprisingly, a single autoregressive neuron can efficiently simulate feedforward circuits, with bounds independent of circuit depth. For the
threshold activation, this yields efficient simulation of any polynomial-time
Turing machine together with efficient learning from complete CoT
supervision. For ReLU, our construction enables the efficient training of a single autoregressive neuron to simulate a large, nontrivial neural networks, as demonstrated by our CIFAR-10 experiment.
Together, these results show that a simple next-token predictor can express complex computation, while
learning tractability can depend critically on exposing high-quality intermediate computation.

Our results suggest that the quality of chain-of-thought supervision should be understood relative to the next-token architecture: 
intermediate steps should expose computations that the predictor can represent and learn. 
Our construction demonstrates an extreme case of this principle, where appropriately structured computational traces enable a single autoregressive neuron to efficiently learn to simulate complex computations. 
This suggests a tradeoff between the structure of intermediate supervision and the architectural complexity required of the next-token predictor. 
Although traces in existing text corpora are difficult to control, teacher-generated traces offer an opportunity to shape this supervision. 
This connects our results to the success of reasoning distillation, in which relatively simple students acquire reasoning capabilities through supervision from more complex teachers \citep{ho2023large,guo2025deepseek}. 
While our results do not establish parameter compression, 
they provide an idealized theoretical basis for this practice: 
exposing suitable intermediate computations can make a complex task learnable by a much simpler next-token predictor. 
This perspective motivates jointly designing teacher-generated traces and the student's architecture. \newline

Several theoretical questions remain open for the current architecture. Motivated by \Cref{sec:lower-bounds}, can
simulations without direct edge alignment improve the worst-case $s^2$ terms in
generation length? Can structured or sparse circuits instead
admit non-trivial bounds which are governed by their edge count alone? Beyond the current architecture, is
there a natural and simple extension of the autoregressive neuron that achieves
the optimal linear generation length $T=\Theta(s)$ while preserving time
invariance and the CoT-preserving property? Could a stronger but still
simple architecture achieve
$T=o(s)$? Such an improvement would necessarily move beyond gate-by-gate
generation.

\subsection*{Acknowledgments}
NJ is supported by the Apple AI/ML PhD fellowship. GV is supported by the Israel Science Foundation (grant No. 2574/25), by a research grant from Mortimer Zuckerman (the Zuckerman STEM Leadership Program), and by research grants from the Center for New Scientists at the Weizmann Institute of Science, and the Shimon and Golde Picker -- Weizmann Annual Grant.

\subsection*{AI use statement}
The core theoretical insights behind our results, including the proof technique described in \Cref{subsec:proof-technique}, as well as the initial versions of the proofs, were developed entirely by the authors without assistance from AI tools.
We used OpenAI models to edit the wording and
presentation of the paper, including the presentation
of proofs, identify relevant literature, and assist
with writing code for the experiments.
The authors reviewed the AI-assisted edits, verified
the cited references, and tested the experimental code.
We take responsibility for the final content of this work.

\begingroup
\emergencystretch=2em
\printbibliography[title={References}]

@article{erdos1941problem,
  title={On a problem of Sidon in additive number theory, and on some related problems},
  author={Erdos, Paul and Tur{\'a}n, P{\'a}l},
  journal={J. London Math. Soc},
  volume={16},
  number={4},
  pages={212--215},
  year={1941}
}

@inproceedings{joshi2025theory,
  title={A Theory of Learning with Autoregressive Chain of Thought},
  author={Joshi, Nirmit and Vardi, Gal and Block, Adam and Goel, Surbhi and Li, Zhiyuan and Misiakiewicz, Theodor and Srebro, Nathan},
  booktitle={The Thirty Eighth Annual Conference on Learning Theory},
  pages={3161--3212},
  year={2025},
  organization={PMLR}
}

@article{malach2023auto,
  title={Auto-regressive next-token predictors are universal learners},
  author={Malach, Eran},
  journal={arXiv preprint arXiv:2309.06979},
  year={2023}
}

@article{altabaa2025cot,
  title={CoT Information: Improved Sample Complexity under Chain-of-Thought Supervision},
  author={Altabaa, Awni and Montasser, Omar and Lafferty, John},
  journal={arXiv preprint arXiv:2505.15927},
  year={2025}
}

@inproceedings{li2024chain,
  title={Chain of thought empowers transformers to solve inherently serial problems},
  author={Li, Zhiyuan and Liu, Hong and Zhou, Denny and Ma, Tengyu},
  booktitle={The Twelfth International Conference on Learning Representations},
  year={2024}
}

@article{valiant1984theory,
  title={A theory of the learnable},
  author={Valiant, Leslie G},
  journal={Communications of the ACM},
  volume={27},
  number={11},
  pages={1134--1142},
  year={1984},
  publisher={ACM New York, NY, USA}
}

@article{wei2022chain,
  title={Chain-of-thought prompting elicits reasoning in large language models},
  author={Wei, Jason and Wang, Xuezhi and Schuurmans, Dale and Bosma, Maarten and Xia, Fei and Chi, Ed and Le, Quoc V and Zhou, Denny and others},
  journal={Advances in neural information processing systems},
  volume={35},
  pages={24824--24837},
  year={2022}
}

@inproceedings{Wei+22,
  author       = {Jason Wei and
                  Maarten Bosma and
                  Vincent Y. Zhao and
                  Kelvin Guu and
                  Adams Wei Yu and
                  Brian Lester and
                  Nan Du and
                  Andrew M. Dai and
                  Quoc V. Le},
  title        = {Finetuned Language Models are Zero-Shot Learners},
  booktitle    = {The Tenth International Conference on Learning Representations, {ICLR}
                  2022, Virtual Event, April 25-29, 2022},
  year         = {2022},
}

@article{hanneke2026sample,
  title={Sample Complexity of Autoregressive Reasoning: Chain-of-Thought vs. End-to-End},
  author={Hanneke, Steve and Mehalel, Idan and Moran, Shay},
  journal={arXiv preprint arXiv:2604.12013},
  year={2026}
}

@article{joshi2026learning,
  title={Learning to Think from Multiple Thinkers},
  author={Joshi, Nirmit and Magen, Roey and Srebro, Nathan and Tsilivis, Nikolaos and Vardi, Gal},
  journal={arXiv preprint arXiv:2604.24737},
  year={2026}
}

@book{arora2009computational,
  title={Computational complexity: a modern approach},
  author={Arora, Sanjeev and Barak, Boaz and others},
  volume={1},
  year={2009},
  publisher={Cambridge University Press Cambridge}
}

@article{pippenger1979relations,
  title={Relations among complexity measures},
  author={Pippenger, Nicholas and Fischer, Michael J},
  journal={Journal of the ACM (JACM)},
  volume={26},
  number={2},
  pages={361--381},
  year={1979},
  publisher={ACM New York, NY, USA}
}

@article{goldreich1986construct,
  title={How to construct random functions},
  author={Goldreich, Oded and Goldwasser, Shafi and Micali, Silvio},
  journal={Journal of the ACM (JACM)},
  volume={33},
  number={4},
  pages={792--807},
  year={1986},
  publisher={ACM New York, NY, USA}
}

@article{haastad1999pseudorandom,
  title={A pseudorandom generator from any one-way function},
  author={Haastad, Johan and Impagliazzo, Russell and Levin, Leonid A and Luby, Michael},
  journal={SIAM Journal on Computing},
  volume={28},
  number={4},
  pages={1364--1396},
  year={1999},
  publisher={SIAM}
}

@article{KearnsValiant1994,
  author  = {Michael J. Kearns and Leslie G. Valiant},
  title   = {Cryptographic Limitations on Learning Boolean Formulae and Finite Automata},
  journal = {Journal of the ACM},
  volume  = {41},
  number  = {1},
  pages   = {67--95},
  year    = {1994},
  doi     = {10.1145/174644.174647}
}

@inproceedings{merrill2024expressive,
  title={The expressive power of transformers with chain of thought},
  author={Merrill, William and Sabharwal, Ashish},
  booktitle={International Conference on Learning Representations},
  volume={2024},
  pages={7690--7706},
  year={2024}
}

@article{brosamle2026expressive,
  title={The Expressive Power of Low Precision Softmax Transformers with (Summarized) Chain-of-Thought},
  author={Br{\"o}samle, Moritz and Eckstein, Stephan},
  journal={arXiv preprint arXiv:2605.18079},
  year={2026}
}

@article{li2026optimal,
  title={The Optimal Sample Complexity of Learning Autoregressive Chain-of-Thought},
  author={Li, Zhiyuan},
  journal={arXiv preprint arXiv:2607.07423},
  year={2026}
}

@mastersthesis{krizhevsky2009learning,
  title={Learning Multiple Layers of Features from Tiny Images},
  author={Krizhevsky, Alex},
  school={University of Toronto},
  year={2009},
  url={https://www.cs.toronto.edu/~kriz/learning-features-2009-TR.pdf}
}

@article{kojima2022large,
  title={Large language models are zero-shot reasoners},
  author={Kojima, Takeshi and Gu, Shixiang Shane and Reid, Machel and Matsuo, Yutaka and Iwasawa, Yusuke},
  journal={Advances in neural information processing systems},
  volume={35},
  pages={22199--22213},
  year={2022}
}

@inproceedings{lightman2024let,
  title={Let's verify step by step},
  author={Lightman, Hunter and Kosaraju, Vineet and Burda, Yuri and Edwards, Harrison and Baker, Bowen and Lee, Teddy and Leike, Jan and Schulman, John and Sutskever, Ilya and Cobbe, Karl},
  booktitle={International Conference on Learning Representations},
  volume={2024},
  pages={39578--39601},
  year={2024}
}

@inproceedings{ho2023large,
  title={Large language models are reasoning teachers},
  author={Ho, Namgyu and Schmid, Laura and Yun, Se-Young},
  booktitle={Proceedings of the 61st annual meeting of the association for computational linguistics (volume 1: long papers)},
  pages={14852--14882},
  year={2023}
}

@article{zelikman2022star,
  title={Star: Bootstrapping reasoning with reasoning},
  author={Zelikman, Eric and Wu, Yuhuai and Mu, Jesse and Goodman, Noah},
  journal={Advances in Neural Information Processing Systems},
  volume={35},
  pages={15476--15488},
  year={2022}
}

@article{jaech2024openai,
  title={Openai o1 system card},
  author={Jaech, Aaron and Kalai, Adam and Lerer, Adam and Richardson, Adam and El-Kishky, Ahmed and Low, Aiden and Helyar, Alec and Madry, Aleksander and Beutel, Alex and Carney, Alex and others},
  journal={arXiv preprint arXiv:2412.16720},
  year={2024}
}

@article{grattafiori2024llama,
  title={The llama 3 herd of models},
  author={Grattafiori, Aaron and Dubey, Abhimanyu and Jauhri, Abhinav and Pandey, Abhinav and Kadian, Abhishek and Al-Dahle, Ahmad and Letman, Aiesha and Mathur, Akhil and Schelten, Alan and Vaughan, Alex and others},
  journal={arXiv preprint arXiv:2407.21783},
  year={2024}
}

@article{guo2025deepseek,
  title={Deepseek-r1: Incentivizing reasoning capability in llms via reinforcement learning},
  author={Guo, Daya and Yang, Dejian and Zhang, Haowei and Song, Junxiao and Wang, Peiyi and Zhu, Qihao and Xu, Runxin and Zhang, Ruoyu and Ma, Shirong and Bi, Xiao and others},
  journal={arXiv preprint arXiv:2501.12948},
  year={2025}
}

@article{brown2020language,
  title={Language models are few-shot learners},
  author={Brown, Tom and Mann, Benjamin and Ryder, Nick and Subbiah, Melanie and Kaplan, Jared D and Dhariwal, Prafulla and Neelakantan, Arvind and Shyam, Pranav and Sastry, Girish and Askell, Amanda and others},
  journal={Advances in neural information processing systems},
  volume={33},
  pages={1877--1901},
  year={2020}
}

@article{vaswani2017attention,
  title={Attention is all you need},
  author={Vaswani, Ashish and Shazeer, Noam and Parmar, Niki and Uszkoreit, Jakob and Jones, Llion and Gomez, Aidan N and Kaiser, {\L}ukasz and Polosukhin, Illia},
  journal={Advances in neural information processing systems},
  volume={30},
  year={2017}
}

@article{wei2022statistically,
  title={Statistically meaningful approximation: a case study on approximating turing machines with transformers},
  author={Wei, Colin and Chen, Yining and Ma, Tengyu},
  journal={Advances in Neural Information Processing Systems},
  volume={35},
  pages={12071--12083},
  year={2022}
}

@article{feng2023towards,
  title={Towards revealing the mystery behind chain of thought: a theoretical perspective},
  author={Feng, Guhao and Zhang, Bohang and Gu, Yuntian and Ye, Haotian and He, Di and Wang, Liwei},
  journal={Advances in neural information processing systems},
  volume={36},
  pages={70757--70798},
  year={2023}
}

@article{doron2026stochastic,
  title={Stochastic Autoregressive Learning},
  author={Doron-Arad, Ilan and Mehalel, Idan and Mossel, Elchanan},
  journal={arXiv preprint arXiv:2608.07224},
  year={2026}
}

@article{bartlett2019nearly,
  title={Nearly-tight VC-dimension and pseudodimension bounds for piecewise linear neural networks},
  author={Bartlett, Peter L and Harvey, Nick and Liaw, Christopher and Mehrabian, Abbas},
  journal={Journal of Machine Learning Research},
  volume={20},
  number={63},
  pages={1--17},
  year={2019}
}

@book{goldreich2007foundations_set,
  author    = {Goldreich, Oded},
  title     = {Foundations of Cryptography: A Two-Volume Set},
  publisher = {Cambridge University Press},
  year      = {2007},
  isbn      = {9780521873734},
  address   = {Cambridge, UK}
}

@article{doron2026theory,
  title={A Theory of Online Learning with Autoregressive Chain-of-Thought Reasoning},
  author={Doron-Arad, Ilan and Mehalel, Idan and Mossel, Elchanan},
  journal={arXiv preprint arXiv:2605.06819},
  year={2026}
}
\endgroup

\newpage
\appendix
\tableofcontents
\newpage

\section{Technical Preliminary}\label{sec:technical-preliminary}
First, we present the main definitions used in the simulation results. Sidon sets and Golomb rulers provide the combinatorial
structure for scheduling circuit gates in the autoregressive
simulations, while local Lipschitz continuity is used to control errors in the approximate simulation.
We then recall the E2E and CoT learning settings (\Cref{sec:learning-with-cot}).
\begin{definition}[Sidon set]\label{def:sidon_set}
    A finite set \(S\subset
     \N\) is called a \emph{Sidon set} if whenever $a+b=c+d$ with $a,b,c,d\in S$, then $\{a,b\}=\{c,d\}$. Equivalently,
    $$\left|\{a+b : a,b\in S,\ a\leq b\}\right| = \binom{|S|+1}{2}.$$        
\end{definition}

\begin{definition}[Golomb ruler]\label{def:golomb_ruler}
A finite set \(G\subset \N
\) is called a \emph{Golomb ruler} if all pairwise differences are distinct, that is, if \(g_b-g_a=g_d-g_c\) with \(g_a,g_b,g_c,g_d\in G, g_a<g_b, g_c<g_d\), then \(\{g_a,g_b\}=\{g_c,g_d\}\).
Equivalently,
\[
\left|\{g_b-g_a : g_a,g_b\in G,\ g_a   < g_b\}\right| = \binom{|G|}{2}.
\]                
\end{definition}
A function $\sigma:\R\to\R$ is \emph{locally Lipschitz} if around
every point in its domain, it behaves like a Lipschitz function.
\begin{definition}[Locally Lipschitz function]\label{def:locally-lipschitz}
A function $\sigma:\R\to\R$ is \emph{locally Lipschitz} if,
for every $z\in\R$, there exist $\delta>0$ and $L<\infty$ such that
$|\sigma(u)-\sigma(v)|\leq L|u-v|$ for all
$u,v\in(z-\delta,z+\delta)$.
\end{definition}

\subsection{Supervised Learning with CoT.}\label{sec:learning-with-cot}
We recall the learning framework of \citet{joshi2025theory}.
Fix $\Sigma=\{0,1\}$, an input length $n$, a generation length $T$,
and a base class $\cF_d\subseteq\Sigma^{\Sigma^*}$.
For a distribution $\cD$ over $\{0,1\}^n$ and a target generator
$f_\star\in\cF_d$, both learning settings aim to predict the
final output $\ete{f_\star}{T}(\vx)$. They differ
only in the information available during training.  An E2E example is $\bigl(\vx,\ete{f_\star}{T}(\vx)\bigr)$,
whereas a CoT example is $\bigl(\vx,\CoT{f_\star}{T}(\vx)\bigr)$.
Thus, CoT supervision reveals all intermediate outputs, while E2E supervision
treats them as latent variables.

\begin{definition}[Realizable E2E learnability]\label{def:e2e-learnability}
We say that $\ete{\cF_d}{T}$ is E2E-learnable over $\{0,1\}^n$ with generation
length $T$ and sample complexity $m(\varepsilon,\delta)$ if there is a
learning rule $A$ such that the following holds.  For every distribution
$\cD$ over $\{0,1\}^n$, target generator $f_\star\in\cF_d$, and
$\varepsilon,\delta\in(0,1)$, let
\begin{align*}
S_{\mathrm{E2E}}
    =\bigl((\vx_1,\ete{f_\star}{T}(\vx_1)),\ldots,
           (\vx_m,\ete{f_\star}{T}(\vx_m))\bigr),
\qquad
\vx_1,\ldots,\vx_m\stackrel{\mathrm{i.i.d.}}{\sim}\cD,
\end{align*}
where $m=m(\varepsilon,\delta)$.  Then, with probability at least
$1-\delta$ over the sample and the randomness of $A$, the predictor
$h=A(S_{\mathrm{E2E}}):\{0,1\}^n\to\{0,1\}$ satisfies
\begin{align*}
\Pr_{\vx\sim\cD}\left[h(\vx)\neq\ete{f_\star}{T}(\vx)\right]
    \leq\varepsilon.
\end{align*}
\end{definition}
A CoT learner (\Cref{def:cot-learnability}) receives examples
$\bigl(\vx_i,\CoT{f_\star}{T}(\vx_i)\bigr)$
but is evaluated on the same final output.

We allow both learners to be improper.  

\textbf{The Computational Aspect.} We call a learner efficient
if its sample complexity and running time are polynomial in
$n,d,T,1/\varepsilon$, and $\log(1/\delta)$.

The main benefit of CoT supervision is computational.  Let $v=\VCdim(\cF_d)$ denote the VC dimension of the base class.
\citet{joshi2025theory} showed that the $\Cons$ rule learns the
autoregressive model under CoT supervision with sample complexity
\[
m_{\mathrm{CoT}}(\varepsilon,\delta)
    = O\left(
        \frac{
            v\log 2T\log(2/\varepsilon)+\log(1/\delta)
        }{\varepsilon}
      \right).
\]
\textbf{Linear Threshold.}
For $\sigma(z)=\ind[z\geq0]$ and $\cF_d=\cFlin$, the predictors
are linear thresholds on binary contexts with real weights and biases.
Their VC dimension is $v=d+1$, so this bound is nearly linear
in the number of parameters and depends only logarithmically
on the generation length.\footnote{\citet{li2026optimal} subsequently removed the dependence
on $T$, obtaining the optimal rate
$O((d+\log(1/\delta))/\varepsilon)$.}
Moreover, the $\Cons$ rule is computationally efficient in this setting:
the observed CoT trajectories provide $T$ labeled prefix--next-token examples,
and a linear threshold consistent with all these examples can be found by
linear programming in polynomial time.

Consequently, Corollary~3.7 and Corollary~4.3 of \citet{joshi2025theory} give the following result.
\begin{corollary}\label{thm:LT-learnable-with-COT}
Autoregressive linear-threshold models (i.e. $\ete{\cFlin}{T}$) are efficiently learnable with CoT
supervision (\Cref{def:cot-learnability}).
\end{corollary}

\section{Related Work}\label{app:related-work}
This section discusses additional related work and provides further details on the studies discussed in the main text.

Chain-of-Thought (CoT) was initially introduced as a prompting approach that encourage a series of intermediate reasoning steps before producing the final answer \citep{wei2022chain,kojima2022large}.
Subsequent work demonstrated that supervising the correctness of these intermediate reasoning steps can further improve reasoning performance, highlighting the benefits of training with CoT supervision \citep{lightman2024let, ho2023large, zelikman2022star}.
More recently, reinforcement learning (RL) techniques applied in LLM training have encouraged these models to generate more and more tokens during inference, which can be interpreted as a form of CoT reasoning \citep{jaech2024openai,guo2025deepseek}.
The Chain-of-Thought supervision we analyze in this work is closely related to the practice of supervised fine-tuning~\cite{brown2020language,Wei+22,grattafiori2024llama} of LLMs.

\paragraph{Expressivity with CoT.} 
Since Transformers \citep{vaswani2017attention} are the dominant architecture underlying LLMs, growing attention has been devoted to characterizing their expressivity \citep{wei2022statistically}, 
particularly when augmented with CoT \citep{merrill2024expressive,brosamle2026expressive, feng2023towards}. 
For example,\citet{merrill2024expressive} showed that the number of autoregressive steps acts as a computational resource for Transformers and, in particular, that polynomially many steps suffice to recognize every language in $\mathsf{P}$. 
Our results establish an analogous phenomenon for autoregressive linear-threshold models in the nonuniform setting: additional steps enable more complex computations, and polynomially many steps allow even a single linear-threshold neuron to capture $\mathsf{P}/\mathsf{poly}$, 
the nonuniform analogue of $\mathsf{P}$. 

\citet{malach2023auto} showed that \emph{time-dependent} AR linear-threshold models can simulate threshold circuits, and hence Turing-machine computations, using a number of autoregressive steps polynomial in the circuit size. 
In related work, \citet{joshi2025theory} established an analogous result for time-invariant AR linear-threshold models, albeit with exponentially many autoregressive steps. 
More recently, \citet{joshi2026learning} showed that these models can efficiently simulate depth-$2$ threshold circuits and used this simulation to establish hardness results for learning from CoT traces generated by multiple distinct generators.

\section{Expressivity}\label{app:expressivity}

Both main-text simulation theorems (\Cref{thm:LT-expressivity-tight,thm:boolean-LT-expressivity}) follow from a common
construction for $\sigma$-circuits. We first give the exact
construction under \Cref{ass:activation}, proving the exact
part of \Cref{thm:LT-expressivity-tight} together with
CoT preservation, efficient construction of the feature map
and designated steps, and preservation of binary inputs
(\Cref{prop:exact-simulation-properties}).
Restricting this construction to threshold activation and
Boolean inputs also proves \Cref{thm:boolean-LT-expressivity}.
We then establish the approximate part of
\Cref{thm:LT-expressivity-tight} by controlling the errors
introduced when undesignated steps produce small nonzero
values (see \Cref{thm:approximate-simulation}).

\subsection{Exact simulation}

We first state the additional properties of the exact
simulation, then prove them jointly with the exact part of
\Cref{thm:LT-expressivity-tight}.
\begin{proposition}[Additional properties of the exact simulation]
\label{prop:exact-simulation-properties}
Under \Cref{ass:activation}, the exact simulation in
\Cref{thm:LT-expressivity-tight} can be chosen to satisfy
the following properties.
\begin{enumerate}
    \item \textbf{CoT preservation.}
    There exist designated steps
    $1\leq t_1<\cdots<t_s=T$, depending only on $n,s$,
    such that, for every size-$s$ $\sigma$-circuit $C$,
    the simulating neuron satisfies
    \begin{align*}
    \ete{f_{\vw,\sigma}}{t}(\phi_{n,s}(\vx))
    =
    \begin{cases}
        C_i(\vx), & \text{if } t=t_i \text{ for some } i\in[s],\\
        0, & \text{otherwise},
    \end{cases}
    \qquad
    \text{$\forall \vx\in[0,1]^n$,$t\in[T]$}.
    \end{align*}

    \item \textbf{Efficiency.}
    The feature map $\phi_{n,s}$ and the designated steps
    depend only on $n,s$ and can be constructed in time
    polynomial in $n,s$. Evaluating $\phi_{n,s}(\vx)$
    takes $O(n')$ real-RAM operations.

    \item \textbf{Binary inputs.}
    The feature map satisfies
    $\phi_{n,s}(\{0,1\}^n)\subseteq\{0,1\}^{n'}$.
\end{enumerate}
\end{proposition}
\begin{proof}[Proof of the exact part of
\Cref{thm:LT-expressivity-tight} and
\Cref{prop:exact-simulation-properties}] Order the non-input gates as $v_1,\ldots,v_s$ such that every gate appears
after all of its non-input predecessors with $v_s$
the output gate. For every $i\in[s]$, let $\theta_i$ be the bias of $v_i$,
let $\vw^{\mathrm{in}}[q,i]$ be the weight from input gate $v_{\mathrm{in},q}$ to $v_i$,
and let $\vw[j,i]$ be the weight from $v_j$ to $v_i$, assigning weight zero to missing edges.

The goal is to choose vector weights $\vw$, a feature map $\phi_{n,s}$, and
time steps $1\leq t_1<\cdots<t_s\leq T$ such that the autoregressive process
satisfies $\ete{f_{\vw,\sigma}}{t_i}(\phi_{n,s}(\vx))=C_i(\vx)$.

Next, partition the input values $\vx[1],\dots,\vx[n]$ into
$k=\lceil n/s\rceil=O(n/s+1)$ blocks, each of size at most $s$. Each block
is encoded in a separate block of $\phi_{n,s}(\vx)$.
Formally, let $k=\lceil n/s\rceil$, and pad $\vx$ by setting $\vx[q]=0$ for
$n<q\leq ks$. For every $r\in[k]$ and $j\in[s]$, set
$\vx_r[j]:=\vx[(r-1)s+j]$ to be the $j$-th element of the $r$-th block. Let $\vw_r^{\mathrm{in}}[j,i]$ be the weight
from $\vx_r[j]$ to $v_i$, that is,
$\vw_r^{\mathrm{in}}[j,i]:=\vw^{\mathrm{in}}[(r-1)s+j,i]$, where weights
from padded coordinates are zero. Define the feature map
$\phi_{n,s}:[0,1]^n\to[0,1]^{n'}$ and the weight vector $\vw\in\R^d$ by
\begin{align*} \label{eq:feature_map_and_weights}
     \vw &= [\underbrace{\vv}_{T \text{ coordinates}} \ \ , \ \underbrace{\vz_k}_{T' \text{ coordinates}},\dots, \underbrace{\vz_1}_{T' \text{ coordinates}},\quad \underbrace{\vz}_{T \text{ coordinates}} ] \\
      \phi_{n,s}(\vx) &=   \hspace{6pt}  [  \underbrace{0^{T-1}, 1}_{T \text{ coordinates}}  ,\underbrace{\vu_k}_{T' \text{ coordinates}} \ ,\dots,\underbrace{\vu_1}_{T' \text{ coordinates}},\quad  \underbrace{0^{T}}_{T \text{ coordinates}} \ ],
\end{align*}
for some $T,T'$ and $d$ that will be specified later. 
The blocks $\vu_r\in[0,1]^{T'}$ and $\vz_r\in\R^{T'}$ for $r\in[k]$ are responsible for the contributions of the input values,
while $\vz$ propagates previously computed gate values. We now describe formally the role and construction of each block:

\paragraph{Construction of $\vv$.}
The vector $\vv$ encodes the biases $\theta_1,\dots,\theta_s$. 
During the first $T$ autoregressive steps, the $1$ in $\phi_{n,s}(\vx)$ shifts across the coordinates of $\vv$. 
We set the coordinate aligned with this $1$ at time $t_i$ to $\theta_i$. 
Thus, at time $t_i$, the bias contribution is $\theta_i$. 
In other time steps $t\neq t_i$, we set the aligned coordinate to a sufficiently negative value to ensure that $\ete{f_{\vw,\sigma}}{t}(\phi_{n,s}(\vx)) = 0$.
Formally, define 
\begin{align*}
M_C:=\sup\{ |C_i(\vx)|:\vx\in[0,1]^n,\ i\in[s]\}<\infty,
\end{align*}
which is well-defined by \Cref{obs:bounded-gate-values}. Set
\begin{align*}
A:=\sum_{r=1}^k\lVert\vz_r\rVert_1+M_C\lVert\vz\rVert_1.
\end{align*}
At every undesignated time, we set the aligned coordinate to $c_\sigma-A-1$, where $c_\sigma$ is from \Cref{ass:activation}.
This suppresses every contribution from the input and previously generated
gate values.
Formally, for every $t\in[T]$, define
\begin{equation}
\vv[-t]
=
\begin{cases}
\theta_i, & t=t_i \text{ for some } i\in[s],\\
c_\sigma-A-1, & \text{otherwise},
\end{cases}
\end{equation}
where here and throughout the proof, negative indices count coordinates from the right endpoint
of a block. Note that after $t$ steps, the $1$ in $\phi_{n,s}(\vx)$ is aligned with the $\ord{t}$ coordinate of $\vv$ from the right.

\paragraph{Construction of $\vu_r$ and $\vz_r$.} The vector $\vu_r$ encodes the input values $\vx_r[1],\dots,\vx_r[s]$ (the $\ord{r}$ block in the partition of $\vx$) separated by zeros.
Their locations are determined by indices $t'_1,\dots,t'_s\in[O(s^2)]$ that will be specified later.
For every $r\in[k]$ and $\ell\in[T']$, define $\vu_r\in[0,1]^{T'}$ by
\begin{equation}
\vu_r[-\ell]
=
\begin{cases}
\vx_r[j], & \ell=t'_j \text{ for some } j\in[s],\\
0, & \text{otherwise}.
\end{cases}.
\end{equation}
That is, $\vx_r[j]$ is located in the $\ord{{t'}_j}$ coordinate of $\vu_r$ from the right.
We can conclude that:
\begin{observation} \label{obs:input_bits_locations}
At step $t_i$, the input value $\vx_r[j]$ is aligned with $\vz_r[-(t_i+t'_j-1)]$ for every $r\in[k]$ and $j\in[s]$.
\end{observation} 
This input value contributes to the preactivation at step $t_i$ with coefficient $\vw_r^{\mathrm{in}}[j,i]$.
Therefore, we define the weight blocks $\vz_r$ as follows: For every $r\in[k]$ and $i,j\in[s]$, set
\begin{align*}
\vz_r[-\ell]
&=
\begin{cases}
\vw_r^{\mathrm{in}}[j,i], & \ell = t_i+t'_j-1,\\
0, & \text{otherwise}.
\end{cases}
\end{align*}
 To ensure that these assignments are consistent we need that the sums $t_i+t'_j$ are all distinct, 
 which will be guaranteed by the choice of $t_i$ and $t'_j$ later in the proof.
 \begin{requirement} \label{req:sums}
    The sums $t_i+t'_j$ for $i,j\in[s]$ are distinct.
\end{requirement}
 
\paragraph{Construction of $\vz$.} The vector \(\vz\) is responsible for propagating the outputs of previously computed gates $C_1(\vx), C_2(\vx), \ldots, C_s(\vx)$. 
More specifically, it encodes the weights between non-input gates.
If at time $t_i$ the process outputs $\ete{f_{\vw,\sigma}}{t_i}(\phi_{n,s}(\vx))=C_i(\vx)$, then at later times $t>t_i$, the value $C_i(\vx)$ shifts through the coordinates of $\vz$, 
and is aligned with $\vz[-(t-t_i)]$. When $t=t_j$ for some $j>i$, we want the previously generated values $C_i(\vx)$ to contribute with coefficients $\vw[i,j]$.
Formally, for every $i\in[s]$ and $\ell \in [T]$, define $\vz\in\R^T$ by
\begin{align*}
\vz[-\ell]  = \begin{cases}
\vw[i,j], & \ell = t_j-t_i \text{ for some } j>i,\\
0, & \text{otherwise}.
\end{cases}
\end{align*}
To ensure that these assignments are consistent we need that the differences $t_j-t_i, j>i$ are all distinct, 
 which will be guaranteed by the choice of $t_i$ later in the proof.  
   \begin{requirement} \label{req:diff}
    The differences $t_j-t_i$ over all $1 \leq i<j \leq s$ are distinct.
\end{requirement}

\paragraph{Construction of $T$, $T'$ and $t_1,\dots,t_s,t_1',\dots,t_s'$.} In this part, we choose the time steps $t_1,\dots,t_s$, the initial input values locations $t_1',\dots,t_s'$, and $T$ to satisfy
 \Cref{req:sums} and \Cref{req:diff}. Fortunately, these requirements are almost exactly the same as those required for a Golomb ruler and Sidon set (See \Cref{def:golomb_ruler} and \Cref{def:sidon_set}).
 Moreover, there is a well-known construction of Golomb rulers and Sidon sets by \citet{erdos1941problem}. Apply \Cref{claim:explicit_sidon_set} with parameter $2s$. Write the resulting
Golomb ruler (Sidon set) as $\{a_1<\cdots<a_{2s}\}$, then
$a_1,\dots,a_{2s}$ are positive, $a_{2s}=O(s^2)$, and their positive
pairwise differences and sums are distinct. For $i,j\in[s]$, set
$t'_j=a_j$, $t_i=a_{s+i}$, and $T=t_s$. Hence
$t_1<\cdots<t_s=T=O(s^2)$, the differences $t_i-t_j$ for $j<i$ are
distinct, and the sums $t_i+t'_j$ are distinct as well. Moreover, we can choose $T' = O(s^2)$ such that $t_i+t'_j\leq T'$ for all $i,j\in[s]$. 
This choice of $T'$ ensures that all non-zero cordinates of $\vu_r$ are aligned with $\vz_r$ throughout the autoregressive process, for any $r\in[k]$.

We prove by induction over the first $T$ steps that the generated value is
$C_i(\vx)$ at step $t_i$ and zero at every other step. Assume this
holds through step $t-1$.

\begin{itemize}
    \item If $t\notin\{t_1,\dots,t_s\}$, then the leading $1$ in
$\phi_{n,s}(\vx)$ is aligned with
$\vv[-t]=c_\sigma-A-1$. Since every input coordinate belongs to $[0,1]$, the
absolute input contribution is at most
$\sum_{r=1}^k\lVert\vz_r\rVert_1$. By the induction hypothesis, every
previously generated value is either zero or some $C_j(\vx)$, so the
absolute propagated contribution is at most $M_C\lVert\vz\rVert_1$.
Thus the preactivation is at most
\begin{align*}
c_\sigma-A-1+\sum_{r=1}^k\lVert\vz_r\rVert_1
+M_C\lVert\vz\rVert_1=c_\sigma-1<c_\sigma.
\end{align*}
The zero-region condition in \Cref{ass:activation} gives
$\ete{f_{\vw,\sigma}}{t}(\phi_{n,s}(\vx))=0$.
\item Suppose that $t=t_i$ for some $i\in[s]$. The leading $1$ in
$\phi_{n,s}(\vx)$ is aligned with $\vv[-t_i]$, so its contribution is
$\theta_i$. Moreover, for every $r\in[k]$ and $j\in[s]$, the value
$\vx_r[j]$ is aligned with $\vz_r[-(t_i+t'_j-1)]$ and contributes
$\vw_r^{\mathrm{in}}[j,i]\vx_r[j]$. A value generated at an earlier step
$\tau<t_i$ is aligned with $\vz[-(t_i-\tau)]$. By the induction hypothesis,
it is zero unless $\tau=t_j$ for some $j<i$. In that case it equals
$C_j(\vx)$, and $\vz[-(t_i-t_j)]=\vw[j,i]$. Consequently, the
preactivation at step $t_i$ is
\begin{align*}
\theta_i
+\sum_{q=1}^n \vw^{\mathrm{in}}[q,i]\vx[q]
+\sum_{j=1}^{i-1}\vw[j,i]C_j(\vx).
\end{align*}
This is precisely the preactivation computed by $v_i$. Applying $\sigma$ gives
$\ete{f_{\vw,\sigma}}{t_i}(\phi_{n,s}(\vx))=C_i(\vx)$, completing the induction.
\end{itemize}

Thus, at every step $t\in[T]$, the generated value is $C_i(\vx)$ if
$t=t_i$ for some $i\in[s]$, and is zero otherwise. This proves the
CoT-preserving property.
Since $t_s=T$ and $v_s$ is the output gate, the final generated value is
$C(\vx)$.

Every coordinate of $\phi_{n,s}(\vx)$ is an input coordinate,
zero, or one. Hence,
$\phi_{n,s}(\{0,1\}^n)\subseteq\{0,1\}^{n'}$.
For threshold activation, every generated token is also binary,
so the simulation on Boolean inputs uses $\Sigma=\{0,1\}$.

Finally, $\phi_{n,s}$ depends only on $n$ and $s$, which determine $T$ and
$T'$ and the initial locations of the input coordinates $t_1',\dots,t_s'$. The preceding construction
therefore constructs a description of $\phi_{n,s}$ in time polynomial in $n$
and $s$. Given this description, evaluating $\phi_{n,s}(\vx)$ consists of
initializing $n'$ coordinates to zero, writing the leading $1$, and
copying the $n$ input values to their prescribed locations. This takes
$O(n')$ real-RAM operations, which proves the efficiency property. Finally, the length of the feature map and the
dimension of $\vw$ are
\begin{align*}
n'=d=T+T'k+T=O(ns+s^2),
\end{align*}
which completes the proof.
\end{proof}

% \begin{proof}[Proof of the exact part of
% \Cref{thm:LT-expressivity-tight}]
\begin{proof}[Proof of \Cref{thm:boolean-LT-expressivity}]
Apply the preceding construction with
$\sigma(z)=\ind[z\geq0]$ and restrict the inputs to
$\{0,1\}^n$. This activation satisfies
\Cref{ass:activation}. Every coordinate of the feature map
is an input coordinate, zero, or one, and every generated
value is binary. Thus the construction gives a simulation
over $\Sigma=\{0,1\}$ with the stated bounds and a
polynomial-time computable feature map depending only
on $n,s$.
\end{proof}

\begin{remark}[General readouts]\label{remark:affine.readout}
Neural networks often use a vector-valued affine output layer
with no activation function.
Although \Cref{thm:LT-expressivity-tight} is stated for
scalar-output $\sigma$-circuits, its construction also supports
networks with such output layers.
Indeed, the exact construction reproduces every gate value
$C_i(\vx)$ at a designated generation step $t_i\in[T]$
by \Cref{property:cot-preservation}..
Consequently, such networks can be simulated by applying
an external affine readout to the designated values
in the generated trajectory.
Alternatively, in an extension of the model that permits
omitting $\sigma$ at designated output steps, the shared
neuron can also compute the affine output coordinates,
one per designated step. Our CIFAR-10 experiments use this designated
affine-output extension to learn all 10 class logits jointly with the hidden
computation.
\end{remark}
% \begin{remark}[Approximate simulation]\label{thm:approximate-simulation}
% The zero-region condition in \Cref{ass:activation} can be relaxed at the cost
% of approximate rather than exact simulation. Suppose instead that $f$ is
% locally Lipschitz and $\lim_{z\to-\infty}f(z)=0$. By choosing the suppressing
% coordinates (the relevant coordinates of $\vv$ from the proof of \Cref{thm:LT-expressivity-tight}) sufficiently negative, the outputs at undesignated steps become
% arbitrarily small, although they need not be zero. These outputs perturb the
% preactivations at later steps. Since the trajectory has finitely many steps,
% local Lipschitz continuity controls the accumulation of these errors. Thus, for
% every fixed $f$-circuit $C$ and every $\varepsilon>0$, the same construction
% approximates every designated gate value and zero at every undesignated step,
% uniformly over $\vx\in[0,1]^n$, to accuracy $\varepsilon$. In particular, the
% final output approximates $C(\vx)$ uniformly to accuracy $\varepsilon$, with
% the same asymptotic bounds on the dimension $d$ and number of steps $T$.

% This relaxed condition is satisfied by common smooth activations such as the
% logistic sigmoid, softplus, GELU, and SiLU (Swish). ReLU and the threshold
% activation already satisfy the stronger zero-region condition in
% \Cref{ass:activation}.
% \end{remark}
\subsection{Approximate simulation}

The following theorem strengthens the approximate part of
\Cref{thm:LT-expressivity-tight} by approximating the entire
simulation trajectory, including all intermediate gate values.

\begin{theorem}[Approximate simulation]
\label{thm:approximate-simulation}
Suppose that $\sigma$ satisfies \Cref{ass:activation-approximate}.
For every input size $n$ and circuit size $s$, use the feature map
$\phi_{n,s}:[0,1]^n\to[0,1]^{n'}$ and designated steps
$1\leq t_1<\cdots<t_s=T$ from the construction in the proof of
\Cref{thm:LT-expressivity-tight}.
For every size-$s$ $\sigma$-circuit $C:[0,1]^n\to\R$ and every
$\varepsilon>0$, there exists $\vw\in\R^d$ such that
\begin{enumerate}
    \item $|\ete{f_{\vw,\sigma}}{t_i}(\phi_{n,s}(\vx))-C_i(\vx)|\leq\varepsilon$, \quad for every $\vx\in[0,1]^n$ and $i\in[s]$.
    \item $|\ete{f_{\vw,\sigma}}{t}(\phi_{n,s}(\vx))|\leq\varepsilon$, \quad for every $\vx\in[0,1]^n$ and
    $t\in[T]\setminus\{t_1,\ldots,t_s\}$.
\end{enumerate}
In particular,
$|\ete{f_{\vw,\sigma}}{T}(\phi_{n,s}(\vx))-C(\vx)|\leq\varepsilon$
for every $\vx\in[0,1]^n$.
The bounds $n',d=O(ns+s^2)$ and $T=O(s^2)$, and the efficiency
properties of the feature map and designated steps, remain unchanged.
The feature map and designated steps are independent of $C$ and
$\varepsilon$.
\end{theorem}

\begin{proof}
Use the construction from the proof of
\Cref{thm:LT-expressivity-tight}, changing only the coordinates of
$\vv$ associated with undesignated steps.
Let $M_C$ be the bound on absolute gate values used in that proof,
and set $B=\lVert\vz\rVert_1$.
Choose $R\geq1$ such that every circuit gate preactivation lies in
$[-R,R]$ for every input.
These bounds exist by \Cref{obs:bounded-gate-values}, since local
Lipschitz continuity implies boundedness on bounded intervals.
Let $L\geq1$ be a Lipschitz constant for $\sigma$ on
$[-R-B,R+B]$, and set
$q=\max\{1,LB\}$ and
$\delta=\min\{1,\varepsilon\} \cdot q^{-T}$.

By the tail condition, there exists $h\in\R$ such that
$|\sigma(z)|\leq\delta$ whenever $z<h$.
At every undesignated step $t$, set
$\vv[-t]=h-\sum_{r=1}^k\lVert\vz_r\rVert_1-(M_C+1)B-1$.
At designated steps, retain $\vv[-t_i]=\theta_i$.
During the first $T$ steps, previously generated values interact
only with $\vz$, so changing these coordinates of $\vv$ does not
change the coefficients propagating errors.

We prove by induction that, through step $t$, every generated value
differs from its prescribed value by at most $\delta q^t$, uniformly
over all inputs. Here the prescribed value is $C_i(\vx)$ at $t_i$
and zero otherwise.
The claim is vacuous at $t=0$.
Suppose it holds through step $t-1$.
Since $\delta q^{t-1}\leq1$, all previously generated values have
absolute value at most $M_C+1$.
At an undesignated step, the chosen coordinate of $\vv$ therefore
makes the preactivation at most $h-1$, so the output has absolute
value at most $\delta\leq\delta q^t$.
At a designated step $t_i$, the construction with prescribed
previous values gives exactly the preactivation of gate $v_i$.
The error in this preactivation is at most
$B\delta q^{t-1}\leq B$.
Both preactivations thus lie in $[-R-B,R+B]$, and the output error
is at most $LB\delta q^{t-1}\leq\delta q^t$.
This completes the induction.

Finally, $\delta q^T\leq\varepsilon$.
Only the suppressing coordinates of $\vv$ have changed, so the
feature map, designated steps, dimension, number of steps, and
stated efficiency properties are unchanged.
\end{proof}

\subsection{Additional Claims}
Here we establish additional claims used in the simulation proofs.
First, we show that circuit gate values are uniformly bounded
over the input domain (\Cref{obs:bounded-gate-values}). We then prove the equivalence between Sidon sets and
Golomb rulers (\Cref{obs:sidon.set.eq.golomb.ruler}) and give an explicit construction of such sets based on \cite{erdos1941problem}.

\begin{observation}\label{obs:bounded-gate-values}
Suppose that $\sigma$ is bounded on every bounded interval. For every
$\sigma$-circuit $C:[0,1]^n\to\R$ there exists an integer $M_C<\infty$ such that $|C_i(\vx)| \leq M_C$ for all $\vx\in[0,1]^n$ and $i\in[s]$, 
where $C_i(\vx)$ is the value of the $i$-th non-input node.
\end{observation}

\begin{proof}
Order the non-input gates as $v_1,\ldots,v_s$ such that every gate appears
after all of its non-input predecessors. We prove inductively that each gate
function $C_i:[0,1]^n\to\R$ is bounded. Suppose this holds for every $j<i$.
The preactivation at $v_i$ is a finite linear combination of input
coordinates and preceding gate functions, with fixed coefficients.
Therefore, it takes values in some bounded interval. Since $\sigma$ is bounded
on every bounded interval, $C_i$ is bounded. Thus, all $C_i$ are bounded,
and we may take
\begin{align*}
M_C
:=
\left\lceil
\sup\{ |C_i(\vx)|:\vx\in[0,1]^n,\ i\in[s]\}
\right\rceil
<\infty.
\end{align*}
\end{proof}

\begin{observation} \label{obs:sidon.set.eq.golomb.ruler}
        A set $S \subset \mathbb{N}$ is a Sidon set if and only if it is a Golomb ruler.    
\end{observation}
\begin{proof}[Proof of \Cref{obs:sidon.set.eq.golomb.ruler}]
    Suppose $S$ is Sidon but not a Golomb ruler. Then two distinct
pairs $a<b$ and $c<d$ satisfy $b-a=d-c$. This implies $b+c=d+a$.
The Sidon property implies either $b=d$ and $c=a$, contradicting
distinctness, or $b=a$ and $c=d$, contradicting the strict inequalities 
The converse follows similarly.
\end{proof}

The following result is based on the classical construction of \citet{erdos1941problem}, originally stated for Sidon sets. Establishing the Sidon property is the main step in the proof. Our formulation additionally records that the resulting set is also a Golomb ruler, can be constructed in polynomial time, and can be translated and restricted to a prescribed size while preserving both properties. For completeness, we include a proof.

\begin{claim}[Explicit Golomb ruler and Sidon set by \citet{erdos1941problem}] \label{claim:explicit_sidon_set}
For every \(n\in\mathbb N\), one can construct in time \(\poly(n)\) a set  \(G\subseteq [1,4n^2]\) of size \(n\) that is both a Golomb
ruler and Sidon set.
\end{claim}

\begin{proof}[Proof of \Cref{claim:explicit_sidon_set}]
We use the classical construction of Sidon sets \citep{erdos1941problem}. Recall that a set \(S\subseteq \mathbb N\) is
\emph{Sidon} if all unordered pairwise sums are distinct, i.e.,
\(a+b=c+d\) with \(a,b,c,d\in S\) implies \(\{a,b\}=\{c,d\}\).

For $n=1$, take $G=\{1\}$, and for $n=2$, take $G=\{1,2\}$. Both sets satisfy the claim.
Assume henceforth that $n\ge3$. Let $p$ be a prime satisfying $n\le p<2n$, which exists by Bertrand's
postulate, and can be found in deterministic time of $O(n \log \log n)$, say by Eratosthenes' sieve. Since $p\ge3$, it is odd.
Consider $S'=\{(x,x^2):x\in\mathbb F_p\}\subseteq\mathbb F_p^2$.
We first note that \(S'\) is Sidon. Indeed, suppose
\((i,i^2)+(j,j^2)=(k,k^2)+(\ell,\ell^2)\). Then we have the next equalities in $\sF_p$:
\begin{align*}
i+j &= k+\ell, \\
i^2+j^2 &= k^2+\ell^2.
\end{align*}
Since $(i+j)^2=i^2+2ij+j^2$ and $2$ is invertible in $\mathbb F_p$, the two equalities imply $ij=k\ell$.
Thus the unordered pairs \(\{i,j\}\) and \(\{k,\ell\}\) have the same sum
and product, hence they are roots of the same quadratic polynomial $x^2-(i+j)x+ij$, and thus \(\{i,j\}=\{k,\ell\}\).
To embed the Sidon set $S'$ into the integers, we use the mapping $f(x, y) = 2px + y$, where $x, y \in \{0,\dots,p-1\}$. We argue that $S:=f(S')$ is a Sidon set in $\sN$. First, note that the mapping $f$ is injective, as each pair $(x, y)$ corresponds to a unique integer representation in base $2p$ and thus $|f(S')|=p\ge n$. To confirm that the Sidon property is preserved, suppose that for $(x_1,y_1), (x_2,y_2), (x_3,y_3), (x_4,y_4) \in S'$, we have $f(x_1,y_1) + f(x_2,y_2) = f(x_3,y_3) + f(x_4,y_4)$. Expanding this using our mapping yields:
\begin{align*}
2p(x_1 + x_2 - x_3 - x_4) = (y_3 + y_4 - y_1 - y_2).
\end{align*}
Since $|y_3 + y_4 - y_1 - y_2| < 2p$ and the left-hand side is a multiple of $2p$, both sides of the equation must be zero, which implies that $x_1 + x_2 = x_3 + x_4$ and $y_1 + y_2 = y_3 + y_4$ simultaneously in $\sN$. These are exactly the conditions defining the Sidon property for $S'$ in $\mathbb{F}_p^2$. Since $S'$ is Sidon, these equalities imply that $\{(x_1, y_1), (x_2, y_2)\} = \{(x_3, y_3), (x_4, y_4)\}$. Finally, because $f$ is injective, this coordinate-wise equality forces the equality of the mapped values: $\{f(x_1, y_1), f(x_2, y_2)\} = \{f(x_3, y_3), f(x_4, y_4)\}$. Thus, the set $S$ is a Sidon set in $\mathbb{N}$.
Finally, define
\begin{align*}
G=\{2px+(x^2\bmod p)+1:0\le x<n\},
\end{align*}
where $x$ ranges over integers and $x^2\bmod p$ denotes the least nonnegative residue modulo $p$.
Since $n\le p$, the set $G$ is a subset of $S+1$ and has exactly $n$ elements by injectivity of $f$.
Translation and taking subsets preserve the Sidon property, so $G$ is Sidon and, by \Cref{obs:sidon.set.eq.golomb.ruler}, a Golomb ruler.
Moreover,
\begin{align*}
\max G\le 2p(n-1)+(p-1)+1=p(2n-1)<4n^2.
\end{align*}
Thus $G\subseteq[1,4n^2]$, and the construction takes $\poly(n)$ time.
\end{proof}

% \section{Proofs from \texorpdfstring{\Cref{sec:linear-thresholds}}{Section}}
\subsection{Proof of \texorpdfstring{\Cref{cor:linear-threshold-universal}}{Corollary}}
\begin{proof}[Proof of \Cref{cor:linear-threshold-universal}]

Set $s_n=\lceil\tau(n)\log(\tau(n))\log(n)\rceil$ and
$\phi=\phi_{n,s_n}$, where $\phi_{n,s_n}$ is from \Cref{thm:boolean-LT-expressivity}.
These choices depend only on $n$ and $\tau$.
Fix a deterministic Turing machine $M$ running in time at most $\tau(n)$.
By \citet{pippenger1979relations}, its computation on inputs of length
$n$ can be simulated by a Boolean circuit $C_n$ consisting of AND,
OR, and NOT gates, of size $O(\tau(n)\log(\tau(n)))$\footnote{See also Sec.~6 and
the proof sketch of Theorem~6.7 in \citet{arora2009computational}.}, where the
implicit constant may depend on $M$.
The circuit satisfies $C_n(\vx)=M(\vx)$ for every
$\vx\in\{0,1\}^n$. For all sufficiently large $n$ depending on $M$,
the factor $\log(n)$ exceeds this constant, so $C_n$ has size
at most $s_n$. Fix such an $n$.

For threshold activation $\sigma(z)=\ind[z\geq0]$, each Boolean gate can be simulated by a threshold gate as follows:
\begin{align*}
x\wedge y &= \sigma(x+y-\tfrac32), &
x\vee y &= \sigma(x+y-\tfrac12), &
\neg x &= \sigma(\tfrac12-x).
\end{align*}
Hence, $C_n$ is a $\sigma$-circuit of size at most $s_n$.
By \Cref{remark:common-feature-map}, we may pad it to size $s_n$
without changing its output. Applying
\Cref{thm:boolean-LT-expressivity} with the feature map
$\phi=\phi_{n,s_n}$ chosen above gives a vector $\vw\in\R^d$ such that
\begin{align*}
\ete{f_{\vw}}{T}(\phi(\vx))
=C_n(\vx)
=M(\vx)
\qquad\text{for every $\vx\in\{0,1\}^n$}.
\end{align*}
The feature map and the generated tokens are binary by
\Cref{thm:boolean-LT-expressivity}, so the entire
autoregressive computation uses $\Sigma=\{0,1\}$.
Moreover,
\begin{align*}
n'=d
&=O(ns_n+s_n^2)
=\widetilde O(n\tau(n)+\tau(n)^2),\\
T
&=O(s_n^2)
=\widetilde O(\tau(n)^2).
\end{align*}
Since $\tau(n)=O(\poly(n))$, we have $s_n=\poly(n)$.
The map $\phi_{n,s_n}$ is computable in time polynomial in $n$
and $s_n$, so $\phi$ is computable in $\poly(n)$ time.
The feature maps and the constants in the dimension and
generation-length bounds are independent of $M$.
\end{proof}

\section{Learning}
This part provides the remaining formal results supporting
the learning consequences discussed in \Cref{sec:cot-learning} of the main text.
We first give the learning procedure and generalization guarantee
for ReLU networks under supervision of all neuron activations (\Cref{sec:learning-relu-NN-with-cot}).
We then prove hardness of E2E learning for autoregressive
linear-threshold models assuming that one-way functions exist (\Cref{sec:e2e-hardness}).

\subsection{Learning Neural Networks with CoT}\label{sec:learning-relu-NN-with-cot}

We consider binary classification by scalar-output ReLU networks,
using $\sigma(z)=\max\{0,z\}$ and $\Sigma=\R$.
For a network $C:[0,1]^n\to\R$, the target label is
$\ind[C(\vx)>0]$.
Each training example reveals an input $\vx$ together with all
neuron activations $C_1(\vx),\ldots,C_s(\vx)$.
By \Cref{property:cot-preservation}, the observed activations suffice
to construct the complete training trajectory of the simulating
autoregressive neuron: place $C_i(\vx)$ at step $t_i$ and zero
at every other step.
The designated steps depend only on $s$, so this conversion
requires no knowledge of the network weights.
In particular, every prefix and its next-token target are known.
\Cref{property:cot-preservation} guarantees that a single
shared ReLU neuron is consistent with all these training pairs.
\paragraph{Finding a consistent neuron.}
Each observed trajectory provides prefix--next-token pairs $(\vz,y)$.
Since all target tokens are nonnegative, consistency of a ReLU neuron
$f_{\vw,b}$ with such a pair is equivalent to
\begin{align*}
    \langle\vw,\vz\rangle+b &= y
        && \text{if } y>0,\\
    \langle\vw,\vz\rangle+b &\leq 0
        && \text{if } y=0.
\end{align*}
Thus, fitting $m$ trajectories of length $T$ reduces to a linear
feasibility problem with $d+1$ variables and at most $mT$ constraints.
CoT preservation guarantees feasibility.
Moreover, consistency on every observed prefix implies, by induction
over the generation steps, that the learned neuron reproduces each
training trajectory under autoregressive generation.
In particular, it predicts every training label correctly.
For rational input data, a feasible solution can be found in time
polynomial in the encoding length of this linear program.

\paragraph{Generalization.}
We bound the complexity of the learned autoregressive model by
unrolling it into a feedforward network.
The bound applies to every choice of the shared parameters, including
solutions that do not correspond to the original teacher architecture.

\begin{claim}\label{claim:ar-relu-generalization}
Fix $d,T\geq1$ and a feature map $\psi$ taking values in $\R^*$
that is independent of the training sample. The classifier class
\begin{align*}
    \left\{
        \vx\mapsto
        \ind\!\left[
            \ete{f_{\vw,b}}{T}(\psi(\vx))>0
        \right]
        : \vw\in\R^d,\ b\in\R
    \right\}
\end{align*}
has VC dimension
$O((d+1)T^2\log(2(d+1)T))$.
Consequently, in the realizable setting, a sufficient number of
i.i.d.\ examples for every consistent classifier to have error at most
$\varepsilon$, with probability at least $1-\delta$, is
\begin{align*}
    m=O\left(
        \frac{
            (d+1)T^2\log(2(d+1)T)\log(2/\varepsilon)
            +\log(1/\delta)
        }{\varepsilon}
    \right).
\end{align*}
\end{claim}

\begin{proof}
Unrolling the autoregressive computation produces a feedforward
network with $T$ ReLU units. Each unit receives at most $d$ inputs
from the initial context and preceding units, and has one bias.
Allowing separate weights and biases at each step gives a larger
class with $W=O((d+1)T)$ parameters and $L=O(T)$ layers.
To match the architecture assumed by \citet{bartlett2019nearly},
we append a fixed identity output unit that copies the final
generated token. This leaves the computed function unchanged.

By the VC-dimension bound of \citet{bartlett2019nearly}, this enlarged
classifier class has VC dimension $O(WL\log(2W))$.
The shared-parameter class is a subclass and satisfies the same
upper bound. Composing with the fixed map $\psi$ cannot increase
VC dimension.
The sample bound follows from the standard realizable
generalization bound for consistent classifiers.
\end{proof}

\begin{corollary}\label{cor:relu-cot-learning}
Let $C:[0,1]^n\to\R$ be a size-$s$ ReLU network, and let $\cD$
be any distribution over $[0,1]^n$.
Given complete neuron activations on sufficiently many i.i.d.\
inputs drawn from $\cD$, one can learn a single autoregressive
ReLU neuron of dimension $d=O(ns+s^2)$ using $T=O(s^2)$
generation steps such that, with probability at least $1-\delta$,
\begin{align*}
    \Pr_{\vx\sim\cD}\left[
        \ind\!\left[
            \ete{f_{\vw,b}}{T}(\phi_{n,s}(\vx))>0
        \right]
        \neq \ind[C(\vx)>0]
    \right]
    \leq\varepsilon.
\end{align*}
The required number of training examples is polynomial in
$n,s,1/\varepsilon$, and $\log(1/\delta)$.
Learning reduces to a linear program of polynomial size;
for rational data, its running time is polynomial in the
encoding length.
\end{corollary}

\begin{remark}[Regression]\label{remark:regression}
The argument also extends to approximating $C(\vx)$ under any bounded
nonnegative loss that is a nondecreasing function of the absolute
prediction error and vanishes at zero. Here we normalize the loss to take
values in $[0,1]$. Examples include bounded absolute-error
loss $\ell(\hat y,y)=\min\{1,|\hat y-y|\}$ and bounded squared-error
loss $\ell(\hat y,y)=\min\{1,|\hat y-y|^2\}$.
Indeed, the bounds of \citet{bartlett2019nearly} also apply to pseudodimension.
For these losses, the induced loss class has pseudodimension bounded
by a constant multiple of that of the predictor class, so standard
uniform convergence bounds yield polynomial sample complexity.
The same linear-programming procedure reproduces every training
output exactly, giving zero empirical loss. Thus, sufficiently
many samples ensure arbitrarily small expected loss.
\end{remark}

\subsection{Hardness of Learning with E2E supervision}\label{sec:e2e-hardness}
We consider $\sigma(z)=\ind[z\geq0]$ and $\Sigma=\{0,1\}$.

Theorem~4.4 of \citet{joshi2025theory} rules out efficient E2E learning under the assumption
that polynomial-size constant-depth threshold circuits are hard to learn.
Their reduction uses dimension and generation length $O(s^L n)$ to simulate a
size-$s$, depth-$L$ threshold circuit.  This overhead is polynomial for
constant $L$, but it becomes exponential when the depth grows with the running
time.  Our simulation has polynomial overhead for arbitrary polynomial-time
computations and therefore yields hardness from the weaker assumption that
one-way functions exist.

We first recall the cryptographic assumption.  A one-way function is easy to
compute but hard to invert on a random input.

\begin{definition}[One-way function]
A function $f\colon\{0,1\}^*\to\{0,1\}^*$ is a one-way function if it is computable in polynomial time and, for every probabilistic polynomial-time algorithm $I$,
\begin{align*}
\Pr_{\substack{\vx\sim\Unif(\{0,1\}^n)\\
               \text{randomness of }I}}
\left[I(f(\vx),1^n)\in f^{-1}(f(\vx))\right]
\leq n^{-\omega(1)}.
\end{align*}
\end{definition}
One-way functions imply pseudorandom generators \citep{haastad1999pseudorandom}, which in turn imply
pseudorandom function (PRF) families \citep{goldreich1986construct}.  Such functions have polynomial-size
circuits. The guarantee is that no polynomial-time algorithm can distinguish a
pseudorandom function from a uniformly random function with non-negligible
probability. Existence of PRF implies hardness of learning \citep{valiant1984theory, KearnsValiant1994}. 
More specifically, an efficient learner with nontrivial prediction advantage would
distinguish a pseudorandom function from a uniformly random function by
testing the learned predictor on fresh inputs, contradicting pseudorandomness.
This sequence of works give the following result:

\begin{theorem}[{\citep{KearnsValiant1994, goldreich1986construct,haastad1999pseudorandom}}]\label{thm:owf_implies_hardness_of_learning}
% [{\citep{KearnsValiant1994, goldreich1986construct,haastad1999pseudorandom}}]
If one-way functions exist, then there are a polynomial $q$ and function
classes $\cG_n\subseteq\{0,1\}^{\{0,1\}^n}$ such that every function in
$\cG_n$ is computable by a Boolean circuit of size at most $q(n)$, but no
probabilistic polynomial-time algorithm weakly learns $(\cG_n)_{n\geq1}$,
even improperly and under the uniform distribution.
\end{theorem}
Combining with \Cref{cor:linear-threshold-universal} gives the following result:
\begin{corollary}\label{cor:e2e_hardness_from_owf}
If one-way functions exist, then autoregressive linear-threshold models are
not efficiently E2E-learnable (\Cref{def:e2e-learnability}), even improperly.
\end{corollary}
\begin{proof}
Suppose that an efficient E2E learner exists.  Let $(\cG_n)_{n\geq1}$ and
$q$ be as in \Cref{thm:owf_implies_hardness_of_learning}, and set
$s=q(n)$. By \Cref{thm:boolean-LT-expressivity}, there is
a single polynomial-time computable map
$\phi_{n,s}:\{0,1\}^n\to\{0,1\}^{n'}$ such that, for every
$g\in\cG_n$, there is a linear threshold
$f_{\vw_g}\in\cFlin$ satisfying
\begin{align*}
    \ete{f_{\vw_g}}{T}(\phi_{n,s}(\vx))=g(\vx)
    \qquad\text{for every $\vx\in\{0,1\}^n$},
\end{align*}
where $n'=O(ns+s^2)$, $d=O(ns+s^2)$, and $T=O(s^2)$.  Crucially,
$\phi_{n,s}$ depends only on $n$ and the common circuit-size bound $s$, not on
$g$.

We use the assumed E2E learner to learn $\cG_n$.  Transform every labeled
example $(\vx_i,g(\vx_i))$ into
$(\phi_{n,s}(\vx_i),g(\vx_i))$ and run the E2E learner with parameters $d$ and
$T$.  The transformed sample is realizable by $f_{\vw_g}$.  If the learner
returns $h$, output the predictor $\vx\mapsto h(\phi_{n,s}(\vx))$.  Its error
under the original input distribution equals the error of $h$ under the
pushforward distribution induced by $\phi_{n,s}$.  Since $s$, $n'$, $d$, and
$T$ are polynomial in $n$, this gives a probabilistic polynomial-time weak
PAC learner for $(\cG_n)_{n\geq1}$, contradicting
\Cref{thm:owf_implies_hardness_of_learning}.
\end{proof}

\section{Proofs from \texorpdfstring{\Cref{sec:lower-bounds}}{Section}}
\begin{proof}[Proof of \Cref{thm:direct-alignment-lower-bound}]
Consider a width-$1$, depth-$L$ $\sigma$-circuit with one non-input
node $v_i$ at each depth $i\in[L]$ and an edge $(v_{i-1},v_i)$
for every $i\in\{2,\dots,L\}$. Choose the weights of these $L-1$
edges to be pairwise distinct. Since the simulation scheme must
handle arbitrary $\sigma$-circuits, it must handle this circuit.

Let $1\leq t_1<\cdots<t_L\leq T$ be the designated steps, and define
$\Delta_i:=t_i-t_{i-1}$ for every $i\in\{2,\dots,L\}$. Direct edge
alignment gives $\vw[-\Delta_i]=w(v_{i-1},v_i)$. Since the edge weights are pairwise distinct, the positive integers
$\Delta_2,\dots,\Delta_L$ are also pairwise distinct. Therefore,
\begin{align*}
T
\geq t_L
=t_1+\sum_{i=2}^{L}\Delta_i
\geq 1+\sum_{q=1}^{L-1}q
=1+\binom{L}{2},
\end{align*}
where the last inequality follows because $\Delta_2,\dots,\Delta_L$ are pairwise distinct.
\end{proof}

\section{Experiments}\label{sec:experiments}

\paragraph{Dataset and supervision split.}
The CIFAR-10 dataset~\citep{krizhevsky2009learning} contains $60{,}000$
$32\times32$ color images across ten classes. We partition the official
$50{,}000$-image training set deterministically, using seed 42, into $45{,}000$
teacher-training images and $5{,}000$ validation images used for teacher
selection. In the saved validation ordering, the first $4{,}000$ images supply
student recovery trajectories, the next $500$ form a holdout from student
recovery, and the last $500$ are unused for student recovery and evaluation.
The recovery holdout belongs to the teacher's validation partition.
The official $10{,}000$-image test set is used only for final evaluation after
freezing all five step-$10{,}000$ checkpoints. It is not used for recovery,
diagnostic tuning, or checkpoint or seed selection.

\paragraph{Teacher architecture.}
The teacher is a Convolutional Neural Network (CNN) consisting of four convolution--batch-normalization--ReLU stages, followed by global average pooling and a final affine readout layer. 
Before compiling the network into the autoregressive format, the batch normalization parameters are folded exactly into the weights and biases of the preceding convolutions. 
The architectural details of the four recovered stages are summarized in \Cref{tab:teacher_arch}. 

\begin{table}[h]
    \centering
    \caption{Architectural details of the teacher's compiled convolutional stages.}
    \label{tab:teacher_arch}
    \small
    \begin{tabular}{@{}lccccc@{}}
        \toprule
        Stage & Channels & Kernel & Stride & Activation Shape &  \#Parameters \\
        \midrule
        1 & $32$  & $5\times5$ & $4$ & $32\times8\times8$ & $2{,}432$ \\
        2 & $64$  & $3\times3$ & $2$ & $64\times4\times4$ & $18{,}496$ \\
        3 & $128$ & $3\times3$ & $2$ & $128\times2\times2$ & $73{,}856$ \\
        4 & $128$ & $3\times3$ & $1$ & $128\times2\times2$ & $147{,}584$ \\
        \bottomrule
    \end{tabular}
\end{table}

These stages yield $4{,}096$ scalar ReLU gates and contain $242{,}368$ learned
parameters. The compiled circuit appends ten affine output gates in class
order. The student learns these coefficients
and ten output biases, adding $1{,}290$ parameters. In total, recovery covers
$4{,}106$ designated gates and $243{,}658$ tied learned parameters.

\paragraph{Teacher training and selection.}
We train the teacher on the $45{,}000$ labeled training images using
cross-entropy and SGD, selecting the checkpoint by validation accuracy.
The frozen teacher achieves $72.15\%$ accuracy on the official test set.

\paragraph{Student construction.}
We construct an autoregressive student with shared weights and a fixed
activation schedule, using the affine-output extension in
\Cref{remark:affine.readout}. At every step other than the ten designated
output steps, the recurrence is
\begin{equation}
    h^{\mathrm{stu}}_t
    =\relu\!\left(\left\langle w_{\mathrm{stu}},z_t\right\rangle\right),
    \label{eq:recurrence}
\end{equation}
where $w_{\mathrm{stu}}$ is shared across steps and $z_t$ is the shifting
window containing the encoded image and preceding student outputs.
At the ten designated output steps, we omit $\relu$ in
\eqref{eq:recurrence} to produce signed logits.

Following the construction in
\Cref{thm:LT-expressivity-tight}, the simulation use a shared neuron parametrize by $\vw$ that contain many structural zeros along the weights of the teacher's circuit.
Therefore our goal here is to learn the weights of the teacher's circuit and not these zeros.
We use the Sidon construction
from \Cref{claim:explicit_sidon_set} to assign each circuit edge and bias
to a unique coordinate of $w_{\mathrm{stu}}$. The learner is given the
architecture, sparsity pattern, and convolutional parameter tying, but not the teacher's weights.
That is, the coordinate assignment is independent of the teacher's weight values.
We learn the associated coefficients from activation traces and store
them in a sparse coordinate map, leaving structural zeros implicit.

In more details, the weights that force undesignated outputs to zero all have the
same value. We store this value once and identify the corresponding
coordinates from the designated-step schedule. For each final learned student, we bound its outputs at all
designated steps over inputs in $[0,1]^n$ by interval propagation,
including the ten output logits. We use the largest absolute value
among the interval endpoints in place of $M_C$ when computing $A$
as in the proof of \Cref{thm:LT-expressivity-tight}. The norms
count every occurrence of a tied coefficient. We set the entries of $w_{\mathrm{stu}}$ that align with the
constant $1$ in the encoded input at undesignated steps to $-A-1$. Inductively, undesignated preactivations are at most $-1$,
while Sidon alignment reproduces the learned gates at designated
steps. Thus, evaluating only designated gates preserves the recurrence
in exact arithmetic. 

Then we can represent $w_{\mathrm{stu}}$ sparsely as a map from occupied coordinates to
learnable parameters, without store or create the elements whose value is zero.

\paragraph{Student training.}
We train five students on the same $4{,}000$ complete teacher trajectories,
using recovery seeds $4042,\ldots,4046$. The seeds vary parameter initialization
and minibatch ordering; the teacher checkpoint, image subsets, and optimization
hyperparameters remain fixed. Tied hidden parameters and output biases are
initialized independently according to $\mathcal N(0,0.02^2)$; the pooled output
coefficients use $\mathcal N(0,0.005^2)$. We use Adam with constant learning
rate $0.001$, minibatch size $64$, and a budget of $10{,}000$ steps.

Training uses designated positions only. The frozen teacher supplies fixed
hidden prefixes, post-ReLU activation targets, and raw logit targets.
For a hidden target $y$ and student preactivation $a$, we minimize
\begin{equation}
    \ell(a,y)=
    \begin{cases}
        (a-y)^2, & y>0,\\
        \relu(a)^2, & y=0.
    \end{cases}
    \label{eq:loss}
\end{equation}
Thus, a positive target determines the corresponding preactivation, whereas a
zero target imposes only the constraint $a\leq0$. For affine output gates,
we use squared error $(a-y)^2$ for every target sign. With fixed teacher prefixes, this objective is convex
in the tied trainable parameters; it need not be strictly convex or identify
the teacher's parameters. Training the student on $4{,}000$ trajectories for $10{,}000$ Adam steps took
roughly $5$ minutes on a single NVIDIA L40S GPU.  

\paragraph{Student evaluation.}
At test time, the student classifies images using its own intermediate
activations, followed by global-average pooling and its learned affine
readout. 

We evaluate all five students on the official test set and report both
test accuracy and agreement with the teacher's predictions. Teacher
agreement is the fraction of test examples on which the student and
teacher predict the same class, irrespective of whether that prediction
is correct. We also report two measures of logit error relative to the
teacher. Both measures compute the root mean squared difference between
student and teacher logits, averaged over all test images and all ten
classes. Logits are the class scores before softmax.
Free-logit RMSE evaluates the complete student computation: the student
generates its own hidden activations and applies global-average pooling
and its learned affine readout. This measure therefore reflects errors
in both the hidden computation and the readout.
Teacher-forced output RMSE instead supplies the teacher's final hidden
activations to the student's pooling and learned affine readout.
Using the same hidden activations as the teacher isolates readout error
from differences in the hidden computation.
Both measures assess agreement with the teacher's logits; test accuracy
remains the measure of classification performance.

Evaluation runs sequentially on a single NVIDIA A40 with batch size $64$.
The teacher uses the same evaluation procedure. 

For each student metric, we report the mean and one sample standard
deviation across the five students,
\[
s=\sqrt{\frac{1}{N-1}\sum_{i=1}^{N}(x_i-\bar{x})^2},
\qquad N=5.
\]
For percentage metrics, the standard deviation is in percentage points.
It measures variation across students evaluated on the same test images;
it is not a confidence interval over newly sampled images or dataset splits.

\begin{table}[h]
    \centering
    \caption{Official-test evaluation of all five final students trained on
    $4{,}000$ recovery trajectories. Student metrics show mean $\pm$ sample standard deviation.}
    \label{tab:progression}
    \small
    \begin{tabular}{@{}lc@{}}
        \toprule
        Metric & Official test set \\
        \midrule
        Student accuracy & $(72.106 \pm 0.054)\%$ \\
        Teacher accuracy & $72.150\%$ \\
        Teacher agreement & $(99.484 \pm 0.070)\%$ \\
        Free-logit RMSE & $0.023852 \pm 0.004095$ \\
        Teacher-forced output RMSE & $0.004897 \pm 0.002957$ \\
        \bottomrule
    \end{tabular}
\end{table}

\paragraph{Experimental conclusion.}
With $4{,}000$ recovery trajectories, the student jointly learns the hidden
convolutional parameters and affine output head. Across five final
checkpoints, it achieves $72.106\%$ mean test accuracy, compared with the
teacher's $72.150\%$, and $99.484\%$ mean teacher agreement. These results support our theoretical
finding that complete activation supervision enables a single shared
autoregressive neuron to recover the behavior of a multilayer network without
access to its weights.
\end{document}